\documentclass[12pt]{article} % for the arXiv version

\makeatletter
\newif\ifarxiv
\@ifclassloaded{journal}{\arxivfalse}{\arxivtrue}
\makeatother

\usepackage{enumerate}
\usepackage{multicol}
\usepackage{etex}
\usepackage[utf8]{inputenc}
\usepackage{amsfonts}
\usepackage{xcolor}
\usepackage{comment}
\usepackage{amsmath}
\usepackage{amssymb}
\ifarxiv  \usepackage{amsthm} \fi
\usepackage{mathrsfs}
\usepackage{stmaryrd}
\usepackage{color}
\usepackage[english]{babel}
\usepackage{fontenc}
\usepackage{url}
\usepackage{graphicx}
\usepackage{hyperref}
\usepackage{caption}
\usepackage{epstopdf}
\usepackage{hyphenat}
\usepackage{float}
\usepackage{tikz-cd}
\usetikzlibrary{shapes,arrows,fit,automata}
\usetikzlibrary{positioning}
\usetikzlibrary{decorations.pathreplacing}

\usepackage{algorithm}
\usepackage[noend]{algpseudocode}

\makeatletter
\def\algbackskip{\hskip-\ALG@thistlm}
\makeatother

\usepackage[export]{adjustbox}
\usepackage{tikz}
\usepackage{color}
\usepackage[all]{xy}
\usetikzlibrary{matrix,arrows,decorations.pathmorphing}
\usepackage{booktabs}
\usepackage{array}
\usepackage{nicematrix}
\usepackage{graphicx}
\usepackage{adjustbox}
\usepackage{soul}

\newcolumntype{P}[1]{>{\centering\arraybackslash}p{#1}}

\newcommand{\rank}{\operatorname{rank}}

\ifarxiv
    \newtheorem{theorem}{Theorem}[section]
    \newtheorem{lemma}[theorem]{Lemma}
    \newtheorem{proposition}[theorem]{Proposition}
    \newtheorem{corollary}[theorem]{Corollary}
    
\fi
\newtheorem{fact}[theorem]{Fact}

\ifarxiv
    \theoremstyle{definition}

    \newtheorem{definition}[theorem]{Definition}
\fi
\newtheorem{example}[theorem]{Example}
\newtheorem{remark}[theorem]{Remark}

\usepackage{nicematrix}
\usepackage{bbm}
\usepackage{bm}

\newcommand{\rev}[1]{\textcolor{black}{{#1}}}

\newcommand{\PP}{\mathbb{P}}
\newcommand{\RR}{\mathbb{R}}
\newcommand{\CC}{\mathbb{C}}
\newcommand{\FS}{\mathcal{M}_{\bm k, \bm s}}
\newcommand{\CFS}{\overline{\mathcal{M}}_{\bm k, \bm s}}

\title{Function Space and Critical Points
\ifarxiv \\ \fi
of Linear Convolutional Networks}

\ifarxiv
\author{Kathl\'en Kohn\thanks{Department of Mathematics, KTH Royal Institute of Technology, Stockholm, Sweden.}, Guido Mont\'ufar\thanks{Departments of Mathematics and Statistics, UCLA, CA, USA; Max Planck Institute for Mathematics in the Sciences, Leipzig, Germany.}, Vahid Shahverdi$^\ast$, Matthew Trager\thanks{Amazon AWS AI Labs, New York, NY, USA. Work done outside of Amazon.} 
  }
\else
\author{Kathlén Kohn, Guido Mont\'ufar, Vahid Shahverdi, Matthew Trager}
\fi

\begin{document}
\maketitle
\begin{abstract}
We study the geometry of linear networks with one-dimensional convolutional layers. 
The function spaces of these networks can be identified with semi-algebraic families of polynomials admitting sparse factorizations. We analyze the impact of the network's architecture on the function space's dimension, boundary, and singular points. We  also describe the critical points of the network's parameterization map.
Furthermore, we study the optimization problem of training a network with the squared error loss.
We prove that for architectures where all strides are larger than one and generic data, the non-zero critical points of that optimization problem are smooth interior points of the function space.
This property is known to be false for dense linear networks and linear convolutional networks with stride one. 
\end{abstract}

\section{Introduction} 

\emph{Linear networks} are artificial neural networks with linear activation functions. Despite only representing linear functions, linear networks have been widely studied as a simplified model for analyzing the behavior of deep neural architectures. Previous work investigated various aspects of linear networks, including the optimization landscape and critical points of the loss function~\cite{NIPS1988_123, 10.1109/72.392248, 
NIPS2016_6112, 
DBLP:journals/corr/LuK17, 
zhou2018critical, 
pmlr-v80-laurent18a, 
zhu2020global, 
geometryLinearNets, 
mehta2022loss, 
bharadwaj2023complex}, 
the dynamics of training~\cite{DBLP:journals/corr/SaxeMG13}, and the convergence of gradient flow \cite{DBLP:journals/corr/abs-1910-05505} and gradient descent \cite{nguegnang2021convergence}. 
In particular,~\cite{geometryLinearNets} provided a detailed analysis of ``pure'' and ``spurious'' critical points, which are critical points arising from the local geometry of function space (i.e., the set of end-to-end linear functions) and the parametrization. 
\rev{In this context we may also mention \cite{NEURIPS2019_a0dc078c}, which studied the geometry of the set of functions represented by networks with polynomial activation.}

\emph{Linear Convolutional Networks (LCNs)} are a type of linear network in which each linear map is a convolution. This requirement imposes linear constraints on the entries of the weight matrices---conditions sometimes known as ``weight sharing'' and ``restricted connectivity.'' Although Convolutional Neural Networks (CNNs) are widely used in computer vision applications, LCNs have not received as much attention as their fully-connected counterparts. In~\cite{dft}, LCNs were studied from the perspective of the implicit bias of local parameter optimization. That work, however, considered only non-local convolutions with filters of maximal size. More recent works have studied the effect on the function space of regularizing the parameters of the network, obtaining results for stride-one LCNs with arbitrary filter size fixed across layers \cite{dai2021representation} 
and for certain two-layer multi-channel LCNs \cite{pmlr-v119-pilanci20a,
pmlr-v178-jagadeesan22a}.  
Most closely related to our present work is \cite{LCN}, which studied the geometry of the function space  represented by LCNs for varying filter size sequences. That work showed that the function space of LCNs can be viewed as a semi-algebraic set consisting of polynomials admitting certain factorizations. Several theoretical results on the function space were presented, including a characterization of the boundary and its singularities for LCNs with stride one. For convolutions of higher strides, it was shown that the function space is always contained in a lower-dimensional algebraic set, although many questions remained open in that case. 

In this paper, we aim to fill this gap by studying the geometry of the set of functions represented by LCNs with arbitrary strides. We consider networks with an arbitrary number of layers of one-dimensional convolutions having arbitrary filter sizes. Our main results are a characterization of the dimension, boundary, and singularities of 
the function spaces, as a function of the network's architecture (Theorems~\ref{thm:dim-mu}, \ref{thm:thick-closed}, \ref{thm:sing}, \ref{thm:boundaryProperties}). We also describe the critical points of the parameterization map (Theorem~\ref{thm:crit-gene}). 
Based on that description, we prove the following for architectures where all strides are larger than one:
For generic training data, the non-zero critical points (in parameter space) of the squared error loss correspond to smooth interior points of the function space that are critical points of the loss on that function space (i.e., they are not induced by the network's parametrization map) (Theorem~\ref{thm:exposed}). Our results show that LCNs with arbitrary strides have a rich structure that is manifested in the geometry of certain families of polynomials with structured roots. 
To our knowledge, these polynomial families 
have not been previously explored in detail, and could be of independent interest. 

\rev{We interpret LCNs in terms of reduced LCN architectures (with stride larger than one) composed with stride-one sub-architectures. The reduced architectures can be regarded as defining an initial set of constraints and the stride-one sub-architectures as imposing additional inequality constraints. This can be used in architecture design as it allows us to determine which function spaces are contained in each other and control the inequality constraints by choosing the stride-one sub-architectures.} 

Our analysis on LCNs provides insights into the complex geometric properties of neural networks, some of which may transfer to networks with nonlinear activations. 
Unlike densely connected feedforward linear networks, LCNs have a function space that is semi-algebraic, i.e., it is a manifold with boundary (and singularities). Moreover, unlike the stride-one setting considered in~\cite{LCN}, for arbitrary strides the function space is generally a low-dimensional subset of its ambient space. We believe that both of these qualitative aspects are important features of general neural architectures. Interestingly, increasing the stride size in LCNs also leads to non-linear behavior, since it requires associating filters with polynomials of variables in higher degrees. Geometrically, this means that increasing the stride can be seen as ``twisting'' the function space. As we discuss, the geometry of the function space plays a crucial role in optimization, 
since boundary points and singular points are typically more exposed during training. In LCNs, these special points correspond to functions with ``more structure,'' since they are functions that can be expressed by more restrictive architectures or that can be obtained as compositions of repeated filters. This sort of stratification is also characteristic of general neural networks. The algebraic nature of LCNs enables however a precise and quantitative description of this structure, revealing, for example, unexpected differences in behavior between ``reduced'' (where all strides are larger than one) and ``non-reduced'' architectures. In the future, our analysis could be further extended by considering convolutions of higher dimensions or introducing algebraic activation functions \rev{as in \cite{NEURIPS2019_a0dc078c}}.

\medskip 

This paper is organized as follows. 
In Section~\ref{sec:main-results} we present our main results about LCNs with arbitrary strides. 
We provide proofs of these results in subsequent sections. 
Specifically, in Section~\ref{sec:zariskiClosure} we prove results on the Zariski closure of the function space; in Section~\ref{sec:criticalMu}, we focus on critical points of the parameterization map; in Section~\ref{sec:singular} we analyze the singular points of the function space and in Section~\ref{sec:boundary} we describe its boundary; finally, in Section~\ref{sec:opt} we discuss the optimization of the squared loss.
\rev{We keep track of the major notational concepts in Table \ref{tab:symbols}.}

\section{Main results} 
\label{sec:main-results} 

Linear convolutional networks are 
families of linear maps parameterized as compositions of convolutions. In this work, we focus on single-channel convolutions for one-dimensional signals but we allow arbitrary strides. 
In this setting, \rev{a 
convolution or convolutional layer is associated with a filter $w \in \RR^k$, a stride $s \in \mathbb N$, and an output dimension $d' \in \mathbb{N}$. 
The associated convolution
is a linear map $\alpha_{w,s}: \RR^{d} \rightarrow \RR^{d'}$ with input dimension $d := s(d'-1)+k$}  defined by
\begin{equation}
\label{eq:convolution-definition}
\alpha_{w,s}(\texttt{x})[i] = \sum_{j=0}^{k-1} w[j] \cdot \texttt{x}[{is + j}] \quad  \rev{\text{ for } \texttt{x} \in \mathbb{R}^d \text{ and } i = 0, 1, \ldots, d'-1.}
\end{equation}
% \begin{equation}
% \label{eq:convolution-definition}
% \alpha_{w,s}(\texttt{x})_i = \sum_{j=0}^{k-1} w_j \cdot \texttt{x}_{is + j} \quad  \rev{\text{ for } \texttt{x} \in \mathbb{R}^d \text{ and } i = 0, 1, \ldots, d'-1.}
% \end{equation}
\rev{Note that this relation ensures that $is+j$ ranges from $0$ to $d-1$.}
%Note indices $i$ and $j$ run starting from $0$.
The map~\eqref{eq:convolution-definition} can also be represented as a generalized Toeplitz matrix $T_{w,s} \in \RR^{d' \times d}$.
However, for most of our analysis, we will not be required to specify the input and output dimensions of convolutions, since compositions of layers can be defined purely in terms of filter vectors and strides. This will be clear from Proposition~\ref{prop:poly-mult} below.

\begin{definition} 
The \emph{function space} $\mathcal M_{\bm k, \bm s}$ of a linear convolutional network (LCN) architecture with filter sizes ${\bm k} = (k_1,\ldots,k_L)$, and strides ${\bm s} = (s_1,\ldots,s_L)$ is the set of all linear maps $\alpha$ that can be expressed as a composition $\alpha=\alpha_L \circ \cdots \circ \alpha_1$, where $\alpha_l$ is a 
convolution of filter size $k_l$ and stride $s_l$. Here $L$ is the number of layers of the LCN.
\end{definition}
In the following, we  assume that $k_l > 1$ for all $l=1,\ldots, L$ (layers with filter size $k_l=1$ yield only scalar multiplication and can be discarded without loss of generality; see Proposition~\ref{prop:poly-mult} below).
Each stride $s_l$ can be an arbitrary positive integer. 
The linear maps $\alpha$ in an LCN function space $\mathcal M_{\bm k, \bm s}$ are convolutions of filter size
$k:=k_1+\sum_{l=2}^L(k_l-1)\prod_{i=1}^{l-1}s_i$ and stride $s:= s_1 \cdots s_L$ (\cite[Proposition~2.2]{LCN}). 
Since each such convolution is uniquely determined by its filter $w \in \mathbb{R}^k$, we can view $\mathcal M_{\bm k, \bm s}$ as a subset of~$\RR^k$.

To study $\mathcal M_{\bm k, \bm s}$, we use the fact that compositions of convolutions can be described using polynomial multiplication. For any positive integer $s$ and filter 
%$w=(w_0,\ldots,w_{k-1}) \in \RR^k$, 
$w=(w[0],\ldots,w[k-1]) \in \RR^k$, 
we consider the polynomial
\begin{equation} 
%\pi_s(w) := w_{0}x^{s(k-1)} + w_{1} x^{s(k-2)}y^s+\cdots+w_{k-2}x^s y^{s(k-2)} + w_{k-1}y^{s(k-1)} \in \mathbb{R}[x^s,y^s]_{k-1} . 
\pi_s(w) := w[0]x^{s(k-1)} + w[1] x^{s(k-2)}y^s+\cdots+w[k-2]x^s y^{s(k-2)} + w[k-1]y^{s(k-1)} \in \mathbb{R}[x^s,y^s]_{k-1} . 
\label{eq:polynomials}
\end{equation} 
The map $\pi_s$ is an isomorphism between $\mathbb{R}^k$ and the vector space $\mathbb{R}[x^s,y^s]_{k-1}$ of all homogeneous polynomials of degree $k-1$ in the variables $(x^s,y^s)$.

\begin{proposition}[\cite{LCN}]\label{prop:poly-mult} The function space of the LCN architecture $(\bm k, \bm s)$ 
% is 
\rev{can be identified with the following subset of $\RR^k$:}
\begin{equation}\label{eq:polynomialFactorization}
\FS=\left\{w \in \RR^k \colon \pi_1(w) = \prod_{l=1}^L \pi_{S_l}
(w_l),
\,\, w_l \in \RR^{k_l}\right\}, \quad \text{ where } S_l := \prod_{i=1}^{l-1} s_i.
\end{equation}
\rev{Here, $\pi_s$ is the map from~\eqref{eq:polynomials}}. Equivalently, $\FS$ is the image of the parameterization map
\begin{equation}
\mu_{\bm k, \bm s}: \RR^{k_1} \times \cdots \times \RR^{k_L} \rightarrow \RR^k, \qquad (w_1,\ldots,w_L) \mapsto \pi_1^{-1}\left(\prod_{l=1}^L \pi_{S_l}(w_l)\right). 
\label{eq:parameterization-def}
\end{equation}
\end{proposition}

In light of this result, we often view $\FS$ as a family of homogeneous polynomials admitting a  sparse factorization as in \eqref{eq:polynomialFactorization}. Note that the final stride $s_L$ has no effect on the function space; for this reason, we assume from now on that $s_L=1$. We say that an LCN architecture is \emph{reduced} if all strides other than $s_L$ are greater than one.

\begin{example}
\label{ex:reduced-(3,2)}
Consider the architecture $\bm{k}=(3,2)$ and $\bm{s}=(2,1)$. Then we have that 
$L = 2$, 
$k = 
5$, 
$S_1 =  
1$, 
$S_2 = 
2$. 
Thus, $\mu_{\bm{k},\bm{s}}\colon \mathbb{R}^3\times \mathbb{R}^2\to \mathbb{R}^5; (w_1,w_2) \mapsto 
\pi_1^{-1}(\pi_2(w_2) \pi_1(w_1))$. 
Writing $w_1=(w_1[0],w_1[1],w_1[2])$ and $w_2=(w_2[0],w_2[1])$, 
the function space $\mathcal{M}_{\bm{k},\bm{s}}$ consists of all $w\in\mathbb{R}^5$ with 
%$\pi_1(w) = \pi_{2}(w_2) \pi_{1}(w_1) = (w_{20}x^{2} + w_{21}y^2) (w_{10}x^{2} + w_{11}x^{1}y^{1} + w_{12}y^{2})$. 
$\pi_1(w) = \pi_{2}(w_2) \pi_{1}(w_1) = (w_{2}[0]x^{2} + w_{2}[1]y^2) (w_{1}[0]x^{2} + w_{1}[1]x^{1}y^{1} + w_{1}[2]y^{2})$. 
Multiplying out the latter expression and collecting the coefficients of individual monomials, we see that 
%$w = ( w_{20}w_{10}, w_{20}w_{11}, w_{20}w_{12}+w_{21}w_{10}, w_{21}w_{11}, w_{21}w_{12} )$. 
$w = ( w_{2}[0]w_{1}[0], w_{2}[0]w_{1}[1], w_{2}[0]w_{1}[2]+w_{2}[1]w_{1}[0], w_{2}[1]w_{1}[1], w_{2}[1]w_{1}[2] )$. 
These are precisely the filters of the end-to-end convolutions represented by products $T_{w_2,s_2} T_{w_1,s_1}$ of two generalized Toeplitz matrices $T_{w_1,s_1}$ and $T_{w_2,s_2}$ 
with filter size and stride pairs $k_1=3,s_1=2$ and $k_2=2,s_2=1$, which in the concrete case of end-to-end functions $\mathbb{R}^5\to\mathbb{R}^1$ take the form 
%$T_{w_1,s_1} = \begin{bmatrix} w_{10}&w_{11}&w_{12}&0&0\\0&0&w_{10}&w_{11}&w_{12}\end{bmatrix}$, 
$T_{w_1,s_1} = \begin{bmatrix} w_{1}[0]&w_{1}[1]&w_{1}[2]&0&0\\0&0&w_{1}[0]&w_{1}[1]&w_{1}[2]\end{bmatrix}$, 
$T_{w_2,s_2} = \begin{bmatrix} w_{2}[0]& w_{2}[1]\end{bmatrix}$. 
According to \cite[Example~4.12]{LCN}, 
\rev{the implicit description of }
the function space for this architecture is given by 
\begin{equation*}
\mathcal{M}_{\bm{k},\bm{s}} = \{\rev{w = }(A,B,C,D,E) : AD^2+B^2E-BCD=0 \text{ and } C^2-4AE \ge 0\} \subset \RR^5.
\end{equation*}
The Zariski closure $\overline{\mathcal{M}}_{\bm{k},\bm{s}}$ of this function space is visualized in Figure \ref{fig:singulAR},
which displays a $3$-dimensional slice.

\begin{figure}[t]
\includegraphics[width=12cm]{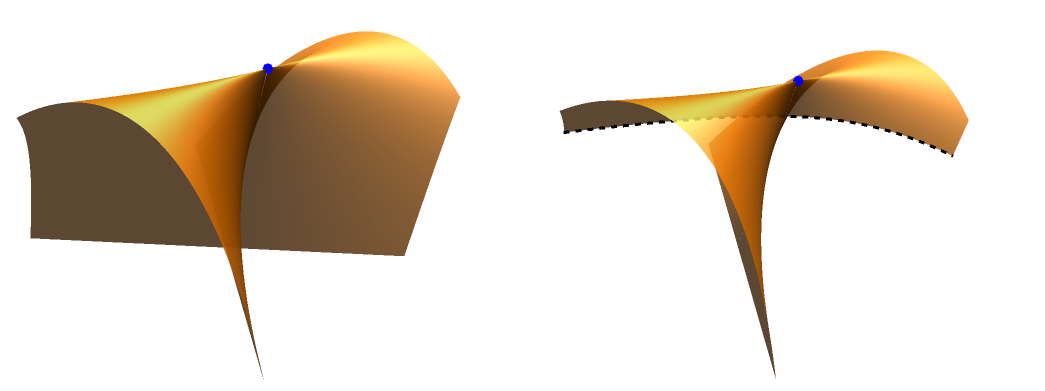}
\centering
\caption{
Left: Slice of the semi-algebraic set $AD^2+B^2E-BCD=0$, $C^2-4AE\ge 0$, obtained by setting $A=1$ and $C=-1$. This set corresponds to the function space $\mathcal{M}_{(3,2),(2,1)}\subseteq \mathbb{R}^5$ in Example~\ref{ex:reduced-(3,2)}.
 Right: The same set intersected with $B^4-4AB(BC-AD)\geq 0$, $D^4-4DE(CD-BE)\ge 0$ and ($AE\le 0$ or $AC\le 0$). This intersection corresponds to the function space $\mathcal{M}_{(2,2,2),(1,2,1)}$ discussed in Example~\ref{ex:k222s121}. The reduced boundary points and the stride-one boundary points are depicted as a blue point and a black dashed curve, respectively; see Theorem \ref{thm:boundaryProperties}.}
\label{fig:singulAR}
\end{figure}
\end{example}
As the previous example suggests, LCN function spaces are semialgebraic sets, that is, they are subsets of $\RR^k$ that are finite unions of solutions sets of finitely many polynomial equalities and inequalities.

\begin{theorem}
\label{thm:dim-mu}
    The LCN function space $\mathcal{M}_{{\bm k},{\bm s}}$ is a semialgebraic Euclidean-closed subset of~$\RR^k$. Its dimension does not depend on the strides and is equal to $k_1+\cdots+k_L- (L-1)$.
\end{theorem}

Our main goal in this work is to investigate how the geometric properties of LCN function spaces are affected by the choice of architecture (the sequences of filter sizes and strides) and to describe how changes in the geometry impact the optimization of a training loss. 
The case of stride-one architectures was studied in detail in~\cite{LCN}. 
In this work, we see that the situation for arbitrary strides is considerably more complex. 
We focus especially on the following basic qualitative features:
\begin{itemize}
\item \emph{thick vs.~thin}: We say that a function space is \emph{thick} if $\dim(\mathcal{M}_{\bm{k},\mathbf{s}}) = k$, that is, if $\FS$ is a full-dimensional semialgebraic subset of its ambient space $\RR^k$ or, equivalently, if its Zariski closure $\overline{\mathcal{M}}_{\bm{k},\bm{s}}$ equals $\RR^k$. We say that the function space is \emph{thin} if $\dim(\FS) < k$, that is, if it is contained in a proper algebraic subset of $\RR^k$.
\item \emph{Zariski closed vs.~non-closed:} A function space is Zariski closed (equivalently, it is an algebraic set) if it can be described using only polynomial equalities. It is Zariski non-closed if its characterization as a semialgebraic set necessarily involves  
polynomial inequalities. 
\item \emph{smooth vs.~singular}: We call a filter $w$ in the function space $\FS$ \emph{singular} if it is a singular point of the algebraic variety $\overline{\mathcal{M}}_{\bm{k},\bm{s}}$.
Otherwise, we say that $w$ is \emph{smooth}.
We say the function space $\FS$ is \emph{smooth} if every filter $w \in \FS$ is smooth.
Note that the Euclidean relative interior of $\FS$ is a manifold. The singular points of that manifold are contained in the singular locus of $\overline{\mathcal{M}}_{\bm{k},\bm{s}}$. Hence, by describing all singular points in the algebraic sense (in Theorem \ref{thm:sing}), we find a superset of all singular points in the manifold sense.
\end{itemize}

These distinctions are relevant for the study of the optimization of training losses for LCNs.
The simplest  
cases are those  
where the function space is both thick and Zariski closed\rev{.
In this case, the LCN function space is equal to its ambient space $\mathbb{R}^k$.
Thus,
minimizing a smooth convex loss function on the function space is simply a convex optimization problem.}

%optimizing the training loss is (almost) a convex optimization problem (all non-global local minima are critical points of the parameterization map). 

\rev{If the function space is not Zariski closed, it has a non-trivial Euclidean (relative) boundary. That boundary might be more exposed during the optimization, in the sense that many critical points of the optimization problem lie on the boundary. This happens for instance for stride-one LCN architectures \cite{LCN}.} 

\rev{Similarly, if the function space is singular, its singular points might be more exposed as well. This happens for dense linear networks \cite{geometryLinearNets,nguegnang2021convergence}.}

As long as the function space is thick, its relevant boundary is the standard boundary in the Euclidean topology on $\mathbb{R}^k$.
If the function space is thin, we need to consider its \emph{relative boundary},
\emph{i.e.}, the set of points in the function space that are limit points of sequences in $\overline{\mathcal{M}}_{{\bm k},{\bm s}} \setminus \mathcal{M}_{{\bm k},{\bm  s}}$.

\begin{table}[ht]
    \centering
    \begin{adjustbox}{max width=\textwidth}
        \begin{tabular}{c|c|c|c|c}
            & \multicolumn{2}{c|}{Zariski closed} & 
            \multicolumn{2}{c}{Zariski non-closed}  \\
            \hline
            Thick &\multicolumn{2}{c|}{$\bm{k}=(3,2), \bm{s}=(1,1)$} &\multicolumn{2}{c}{$\bm{k}=(2,2), \bm{s}=(1,1)$}  \\[1mm] 
            Thin & $\bm{k}=(3,3), \bm{s}=(2,1)$ & $\bm{k}=(2,2), \bm{s}=(2,1)$ &$\bm{k}=(2,2,2), \bm{s}=(1,2,1)$ & $\bm{k}=(3,2), \bm{s}=(2,1)$\\[-3mm] 
            \multicolumn{2}{c}{} &
            \multicolumn{2}{c}{$\underbrace{\hspace{85mm}}_{\text{\normalsize 0 is the only singular point}}$}
        \end{tabular}
    \end{adjustbox}
    \caption{Minimal examples of non-trivial LCN architectures (with at least two layers) of the different geometric types classified in Theorems~\ref{thm:thick-closed} and Theorem \ref{thm:sing}. 
    }
    \label{tab:mtable}
\end{table}

To study these geometric properties, it is useful to decompose an LCN architecture into a reduced architecture and several stride-one sub-architectures.
Intuitively, we reduce an architecture by merging all neighboring layers $l$ and $l+1$ where $s_l=1$ (\emph{i.e.}, $S_{l+1}=S_l$) by multiplying 
their polynomial factors in \eqref{eq:polynomialFactorization}. Formally: 

\begin{definition}
    \label{def:reduction}
Given an LCN architecture $(\bm k, \bm s)$,
we define its \emph{reduced architecture} as $(\tilde {\bm k}, \tilde {\bm s})$, where $\tilde {\bm s}:=(s_{l_1},\ldots,s_{l_{M-1}},1)$ is the subsequence of strides in $\bm s$ that are greater than one (with a final stride equal to one) and
 $\tilde {\bm k}:=(\tilde k_1, \ldots, \tilde k_M)$ with $\tilde k_{j+1} := \sum_{i=l_{j}+1}^{l_{j+1}} (k_i-1)+1$ (and $l_0:=0$, $l_{M}:=L$). 
We  define $M$ \emph{associated stride-one} architectures $(\tilde {\bm k}^j, \bm 1)$, with $\tilde {\bm k}^{j+1} = (k_{l_{j}+1},\ldots, k_{l_{j+1}})$ for $j=0,\ldots, M-1$. 
\end{definition}
Note that if $(\bm k, \bm s)$ was already reduced, then $(\tilde{\bm k}, \tilde{\bm s})=(\bm k, \bm s)$ and $\tilde {\bm k}^j = (k_j)$ for all $1 \leq j \leq L$. 
The parameterization map $\mu_{\bm k, \bm s}$ of the original architecture factorizes as 
\begin{equation}\label{eq:stride-one-reduced-factorization}
\mu_{\bm k, \bm s} 
= \mu_{\tilde {\bm k}, \tilde {\bm s}} \circ  
(\mu_{\tilde {\bm k}^1, {\bm 1}} , \ldots , \mu_{\tilde {\bm k}^M, {\bm 1}}). 
\end{equation}
Here the different arguments $w_1,\ldots, w_L$ of the parametrization map are assigned to their respective stride-one sub-architectures, $\mu_{\tilde{\bm{k}}^{j+1}, \bm{1}} (w_{l_j +1}, \ldots, w_{l_{j+1}})$ for $j=0,\ldots, M-1$, as illustrated below.
\newcommand*{\tikzhl}[1]{\tikz[baseline=(X.base)] \node[fill=black!10,inner sep=2pt, rounded corners] (X) {#1};}
\begin{center}
\begin{tikzpicture}[->,>=stealth',shorten >=1pt,inner sep=.1cm]
\node[text height=.2cm, text depth=0] (W1) {$($\tikzhl{$w_1,\ldots, w_{l_1}$}$,$}; 
\node[text height=.2cm, text depth=0] (W2) [right =0cm of W1]  {\tikzhl{$w_{l_1+1},\ldots,w_{l_2}$}$,$};  
\node[text height=.2cm, text depth=0] (D) [right =0cm of W2] {$\ldots,$}; 
\node[text height=.2cm, text depth=0] (WM) [right=0cm of D] {\tikzhl{$w_{l_{M-1}+1},\ldots, w_L$}$)$}; 

\node (U1) [below = 1cm of W1,text height=.2cm, text depth=0] {$($\tikzhl{$\;\;\;\tilde{w}_1\;\;\;$}$,$}; 
\node (U2) [below = 1cm of W2,text height=.2cm, text depth=0] {\tikzhl{$\;\;\;\tilde{w}_2\;\;\;$}$,$}; 
\node (UD) [below = 1cm of D,text height=.2cm, text depth=0] {$\ldots,$}; 
\node (UM) [below = 1cm of WM,text height=.2cm, text depth=0] {\tikzhl{$\;\;\;\tilde{w}_M\;\;\;$}$)$}; 

\draw[->] (W1) to node[left] {$\mu_{\tilde{\bm{k}}^1, \bm{1}}$} (U1);
\draw[->] (W2) to node[left] {$\mu_{\tilde{\bm{k}}^2, \bm{1}}$} (U2);
\draw[->] (WM)  to node[left] {$\mu_{\tilde{\bm{k}}^M, \bm{1}}$} (UM); 

\node (W) [right = 2cm of WM] {$w$}; 
\draw[->] (WM) to node[above] {$\mu_{\bm{k}, \bm{s}}$} (W); 

\draw[->, bend right=10, out=-10] (UM) to node[below, yshift=-.1cm] {$\mu_{\tilde{\bm{k}}, \tilde{\bm{s}}}$} (W); 

\end{tikzpicture}
\end{center}

Note that the function spaces $\FS$ and $\mathcal M_{\tilde{\bm k}, \tilde{\bm s}}$ of the initial LCN architecture $(\bm k, \bm s)$ and of its reduced architecture $(\tilde{\bm k}, \tilde{\bm s})$  live in the same ambient space $\RR^k$. 
We  show that reducing an architecture can enlarge the function space but does not affect its Zariski closure: 
\begin{equation}
\label{eq:reduction}
    \mathcal{M}_{{\bm k},{\bm s}} \subseteq \mathcal{M}_{\tilde{{\bm k}},\tilde{{\bm s}}} \text{ and } \overline{\mathcal{M}}_{{\bm k},{\bm s}} = \overline{\mathcal{M}}_{\tilde{{\bm k}},\tilde{{\bm s}}}.
\end{equation}
See Lemma \ref{lem:functionSpaceReduction}. The intuition for this fact is that the function spaces of the associated stride-one architectures, which are fed to the reduced architecture, may satisfy constraints in their natural ambient spaces but are always full-dimensional. 
\rev{Conversely, we can interpret \eqref{eq:reduction} as follows: The function space of a reduced architecture with filter sizes $(\tilde k_1,\ldots, \tilde k_M)$ and strides $(\tilde s_1,\ldots, \tilde s_{M-1},\tilde s_{M}=1)$ contains the function space of any architecture with filter sizes $(k^1_1,\ldots, k^1_{m_1},\ldots, k^M_1,\ldots, k^M_{m_M})$ and 
strides $(s^1_1,\ldots, s^1_{m_1},\ldots, s^M_1,\ldots, s^M_{m_M})$ satisfying $\sum_{i=1}^{m_j} (k^j_i-1) + 1 = \tilde k_j$ and $s^j_{m_j}=\tilde s_j$ for $j=1,\ldots, M$ and $s^j_{i}=1$ otherwise. 
%The next Theorem~\ref{thm:thick-closed}c2 implies that 
} 

\begin{example}
\label{ex:k222s121}
Consider the architecture $\bm{k} = (2,2,2)$ and $\bm{s}=(1,2,1)$. Then $M=2$ and $l_1=2$, $l_2=3$. 
The associated reduced architecture is 
$(\tilde{\bm{k}} , \tilde{\bm{s}} ) 
= ( (3,2) , (2,1) ) $, which we already encountered in Example~\ref{ex:reduced-(3,2)}. 
Both architectures appear in Table \ref{tab:mtable}. 
The associated stride-one architectures are 
$(\tilde{\bm{k}^1}, \bm{1}) 
= ((2,2),(1,1))$ 
and 
$(\tilde{\bm{k}^2}, \bm{1}) 
= ((2),(1))$, which have output filters of sizes $
3$ and $2$, respectively, fitting to the domain of the parametrization map of the reduced architecture. 
The function space $\FS$ is 
described by the equality $AD^2+B^2E-BCD=0$ and inequalities $(AE \le 0$ or  $AC \le 0$), $B^4-4AB(BC-AD) \ge 0$, $D^4-4DE(CD-BE) \ge 0$, and $C^2-4AE\ge 0$.
Nonetheless, both architectures have the same Zariski closure. 
This is because one function space is a subset of the other, and they are irreducible and have the same dimension by Theorem \ref{thm:dim-mu}. 
\end{example}

We are now ready to describe the qualitative features of the LCN function space.

\begin{theorem} 
\label{thm:thick-closed}
Let $(\bm k, \bm s)$ be an LCN architecture.
\begin{enumerate} 
    \item[a)] The function space $\mathcal{M}_{\bm k, \bm s}$ is thick if and only if $\bm s = \bm 1$.
    \item[b)] The function space $\mathcal{M}_{\bm k, \bm s}$ is smooth if and only if it is thick.
    \item[c)] To determine whether $\mathcal{M}_{\bm k, \bm s}$ is Zariski-closed:
    \begin{enumerate}
    \item[c1)] If $(\bm k, \bm s)$ is reduced, the function space $\mathcal{M}_{\bm k, \bm s}$ is Zariski-closed if and only if, for all $l = 1, \ldots, L$, $k_l$ is odd or $S_l > \sum_{i=1}^{l-1}  (k_i-1) S_i$.
    \item[c2)] If $\bm s = \bm 1$, then $\mathcal{M}_{\bm k, \bm s}$ is Zariski-closed if and only if at most one of its filter sizes is even. 
    \item[c3)] In general, the function space $\mathcal{M}_{\bm k, \bm s}$ is Zariski-closed if and only if the function spaces of its associated reduced architecture and of its associated stride-one architectures are all Zariski-closed.
    \end{enumerate}
\end{enumerate}
\end{theorem}

\rev{In particular, an LCN is a universal approximator of functions in the natural ambient space $\mathbb{R}^k$ if and only if $\bf s = \bf 1$ and at most one of its filter sizes is even. }

\rev{\begin{remark} The condition $S_l > \sum_{i=1}^{l-1}  (k_i-1) S_i$ from the previous statement will reappear in several other results throughout the paper, e.g., in Theorem~\ref{thm:sing} just below. This relation has a simple interpretation: it means that the composition of the first $l-1$ layers has stride size (equal to $S_l$) at least as large as its filter size (equal to $\sum_{i=1}^{l-1}  (k_i-1) S_i+1$). This in turn implies that the ``receptive fields'' of the output coordinates of this convolution do not overlap; that is, the sets of input coordinates that influence each output coordinate are disjoint. As we will see, this condition also influences the geometry of the projectivized parameterization map of the function space (Remark~\ref{rmk:segre}).
\end{remark}}

% \rev{\begin{example}
% \label{ex:receptive-field-mixed}
% If $\bm k = (3,2)$, $\bm s = (3,1)$, then the first layer has filter size $3$ and stride $3$.....
% \end{example}}

If the function space is not smooth (respectively, not Zariski closed), we aim to understand its singular points (respectively, its Euclidean relative boundary).
Because of \eqref{eq:reduction}, to describe the singular locus of a function space's Zariski closure, it is sufficient to consider reduced architectures.

\begin{theorem} 
\label{thm:sing}
Let $(\bm k, \bm s)$ be a reduced LCN architecture with at least two layers. 
Then the singular locus of the Zariski-closure of the function space is comprised of the zero filter and the union of all 
LCN function spaces 
with the same sequence of strides
whose Zariski closures are proper subsets of $\overline{\mathcal{M}}_{\bm k, \bm s}$: 
\begin{gather*}
{\rm Sing}(\overline{\mathcal{M}}_{\boldsymbol{k}, \boldsymbol{s}}) = \{0\} \cup \bigcup_{\boldsymbol{k}' \in K}\overline{\mathcal{M}}_{\boldsymbol{k}', \boldsymbol{s}} = \{0\} \cup \bigcup_{\boldsymbol{k}' \in K}{\mathcal{M}}_{\boldsymbol{k}', \boldsymbol{s}}, \text{ where }\\
K := K_{\bm k, \bm s}:=\left\{\bm k' \in \mathbb{Z}_{>0}^L \colon 
\begin{array}{l}
    \bm k' \ne \bm k, \,\, \sum_{i=1}^L (k'_i-1) S_i = \sum_{i=1}^L (k_i - 1) S_i,    \\[.1cm]
     \sum_{i=l}^L (k'_i-1) S_i \geq \sum_{i=l}^L (k_i - 1) S_i \,\, \text{ for all } l = 1, \ldots, L
\end{array}
   \right\}.
\end{gather*}
The set $K$ is empty (i.e., $0$ is the only singular point) if and only  if $S_l > \sum_{i=1}^{l-1} (k_i-1) S_i$ for every layer $l$.
\end{theorem}
The conditions for $\bm{k}'$ signify that the truncation of the architecture $(\bm{k'},\bm{s})$ to the first $l-1$ layers has at most the same end-to-end filter size as the corresponding truncation of $(\bm{k},\bm{s})$ and the same filter size when all layers are considered. 
The fact that these are precisely the architectures which satisfy $\overline{\mathcal{M}}_{\boldsymbol{k}',\boldsymbol{s}} \subsetneq \overline{\mathcal{M}}_{\boldsymbol{k},\boldsymbol{s}}$ is shown in Corollary~\ref{cor:strictContainment}. 

We denote the Euclidean relative boundary of the function space by $\partial \FS$.
Since the function space is Euclidean closed (see Theorem \ref{thm:dim-mu}), we have $\partial \FS \subseteq \FS$.  
Recall that $\partial \FS$ consists of all points in $\FS$ that are limits of sequences of points in $\CFS \setminus \FS$.
We distinguish between two types of boundary points, using the fact that the reduced architecture $(\tilde{\bm k}, \tilde{\bm s})$ of $(\bm k, \bm s)$ satisfies 
 $\FS \subseteq {\mathcal M}_{\tilde{\bm k}, \tilde{\bm s}} \subseteq \CFS = \overline{{\mathcal M}}_{\tilde{\bm k}, \tilde{\bm s}}$ as in \eqref{eq:reduction}:
\begin{itemize}
    \item \emph{Reduced boundary points $\partial \FS^R$:} limits in $\FS$ of sequences of points in $\overline{\mathcal M}_{{\bm k}, {\bm s}} \setminus {\mathcal M}_{\tilde{\bm k}, \tilde{\bm s}}$.
    \item \emph{Stride-one boundary points $\partial \FS^S$:} limits in $\FS$ of sequences of points in ${\mathcal M}_{\tilde{\bm k}, \tilde{\bm s}} \setminus \FS$.
\end{itemize}

The boundary of LCN function spaces $\mathcal{M}_{\bm k, \bm 1}$ with stride-one architectures has been fully characterized in terms of the real-root structure of the polynomials in $\pi_1(\mathcal{M}_{\bm k, \bm 1})$ \cite[Proposition 4.4]{LCN}. 
The relative boundary in the case of strides larger than one is significantly more complicated. 
In particular, in contrast to the stride-one case, the reduced boundary points form a semialgebraic set of unexpectedly low dimension (i.e., $\dim \partial \FS^R < \dim \FS - 1$). This can be seen in Figure \ref{fig:singulAR}, for the architecture with $\bm{k}=(2,2,2)$ and $\bm{s}=(1,2,1)$, where the reduced boundary has codimension two while the stride-one boundary has codimension one. 
\begin{theorem}
\label{thm:boundaryProperties}
Let $(\bm k, \bm s)$ be an LCN architecture with reduced architecture $(\tilde{\bm k}, \tilde{\bm s})$.
\begin{enumerate}[a)]
\item Reduced boundary points are on the relative boundary of the reduced architecture; in fact $\partial \FS^R = \partial {\mathcal M}_{\tilde{\bm k}, \tilde{\bm s}} \cap \FS$.
\item Reduced boundary points are contained in 
lower-dimensional  
LCN function spaces; more precisely, we have  
$\partial \FS^R \subseteq  \bigcup_{\tilde{\bm k}' \in \tilde{K}} \mathcal{M}_{\tilde{\bm k}', \tilde{\bm s}} \subseteq  {\rm Sing}(\overline{\mathcal{M}}_{\bm k, \bm s})$, where $\tilde{K} := K_{\tilde{\bm k}, \tilde{\bm s}}$ is defined in Theorem \ref{thm:sing}.
\item The dimension of $\partial \FS^R$ is at most $\dim \FS - \min\{s_i: s_i > 1\}$; in particular, its relative co-dimension is strictly larger than $1$.
\end{enumerate}
\end{theorem}

Our discussion until this moment has focused on the implicit geometry of the function space. In practice, we are also interested in the parameterization of this space by the network's parametrization map. 
Critical points of  
a loss function can in fact arise from degenerate points of the parameterization (called ``spurious critical points'' in~\cite{geometryLinearNets}). We characterize these points in the following result. 

\begin{theorem}
\label{thm:crit-gene}
Let $L>1$.
A filter tuple $\theta=(w_1,\ldots,w_L)$ is a critical point of the parametrization map $\mu_{{\bm k},{\bm s}}$ if and only if there exists a layer $l \in \{1,\ldots,L\}$ such that $w_l=0$ or the polynomials $\pi_{S_l}(w_l)$ and $ \pi_{S_{l-1}}(w_{l-1})\cdots \pi_1(w_1)$ have a non-trivial common factor of the form $Q(x^{S_l}, y^{S_l})$.
In particular, critical values \rev{$w$ correspond to polynomials $\pi_1(w)$ in the discriminant hypersurface, that is, the set of polynomials with a double root.}
\end{theorem}

Finally, we investigate the minimization of the squared error on LCNs. 
 Given some training data $\mathcal D = \{(\texttt{x}^{(i)},\texttt{y}^{(i)})\in \RR^{d_0} \times \RR^{d_L} \colon i=1,\ldots,N\}$, 
 the squared error loss on the function space is 
$$
\ell_{\mathcal D}(w):=\sum_{i=1}^N \|\texttt{y}^{(i)} - T_{w,s} \texttt{x}^{(i)}\|^2,
$$
where $T_{w,s} \in \RR^{d_L \times d_0}$ is the generalized Toeplitz matrix associated with a filter $w \in \RR^k$ and with stride $s$ (note that $d_0 = s(d_L-1)+k$).
When training an LCN with data $\mathcal{D}$, we minimize the squared error loss $\mathcal{L}^\mathcal{D}_{\bm k, \bm s} := \ell_\mathcal{D} \circ \mu_{\bm k, \bm s}$ on the parameter space. Commonly we  use gradient descent to minimize this objective function and thus we are interested in its critical points.
\begin{theorem}
\label{thm:exposed}
Let $N \geq k$. 
For 
almost all\footnote{For all points except those contained in some proper algebraic subset of $(\RR^{d_0} \times \RR^{d_L})^N$.} 
data $\mathcal{D} \in (\RR^{d_0} \times \RR^{d_L})^N$, every critical point $\theta$ of $\mathcal{L}^{\mathcal{D}}_{\bm k, \bm s}$ satisfies one of the following:
    \begin{enumerate}
        \item $\theta$ is a regular point of $\mu_{\bm k, \bm s}$ and $\mu_{\bm k, \bm s}(\theta)$ is a smooth, relative interior point of $\mathcal{M}_{\bm k, \bm s}$ (i.e., $\mu_{\bm k, \bm s}(\theta) \notin \mathrm{Sing}(\overline{\mathcal{M}}_{\bm k, \bm s})$ and $\mu_{\bm k, \bm s}(\theta) \notin \partial \FS$), or
        \item  $\mu_{\bm k, \bm s}(\theta)=0$, or
        \item $\theta$ is a critical point of  
        $(\mu_{\tilde {\bm k}^1, {\bm 1}}, \ldots , \mu_{\tilde {\bm k}^M, {\bm 1}})$. 
    \end{enumerate}
\end{theorem}
Note that the last condition is not possible for reduced architectures. Hence, for reduced architectures, every critical point maps either to zero in the function space $\FS$ or to a smooth interior point of $\FS$ that is a critical point of $\ell_\mathcal{D} |_{\FS}$. In the language of~\cite{geometryLinearNets}, these critical points are ``pure,'' in the sense that they are critical points in function space, rather than being degenerate points of the parameterization map. We remark that in the case of stride-one LCN architectures critical points 
frequently correspond to functions located on the boundary of the function space 
or even  to functions situated in the interior of the function space as spurious (i.e., non-pure) critical points that are induced by the parametrization map; see \cite[Example 5.10]{LCN}. This illustrates the surprising qualitative differences between reduced and stride-one architectures.

\begin{table}[ht]
    \centering \rev{
    \begin{tabular}{cp{11cm}}
        \toprule
        \textbf{Notation} & \textbf{Description} \\
        \midrule
        ${\bm k}=(k_1,\ldots,k_L)$ & sequence of filter sizes for each layer $l=1,\ldots,L$\\[.1cm]
        ${\bm s}=(s_1,\ldots,s_{L-1},1)$ & sequence of strides for each layer $l=1,\ldots,L$ with  $s_L=1$\\[.1cm]
        $({\bm k}, {\bm s})$ & LCN architecture with filters $\bm k$ and strides $\bm s$\\[.1cm]
        $S_l$ & shorthand for $\prod_{i=1}^{l-1}s_i$\\[.1cm]
        $\mathcal M_{\bm k, \bm s}$ & function space of a LCN architecture $(\bm k, \bm s)$, as a subset of $\RR^k$ with $k=\sum_{l=1}^L (k_i-1){S_l} + 1$ \\[.2cm]
        $\pi_s$ & polynomial coefficient map $\RR^k \rightarrow \mathbb{R}[x^s,y^s]_{k-1}$; see~\eqref{eq:polynomials}\\[.1cm]
       $\mu_{\bm k, \bm s}$ & LCN parameterization map $\prod_{i}^L \RR^{k_i} \rightarrow \RR^k$; see~\eqref{eq:parameterization-def}\\[.15cm]
    $(\tilde {\bm k}, \tilde {\bm s})$ & reduced architecture associated with $({\bm k}, {\bm s})$; see Definition~\ref{def:reduction}\\[.15cm]
    $\overline{\mathcal M}_{\bm k, \bm s}$ & Zariski closure of ${\mathcal M}_{\bm k, \bm s}$\\[.1cm]
    $\partial{\mathcal M}_{\bm k, \bm s}$ & Euclidean relative boundary of ${\mathcal M}_{\bm k, \bm s}$ \\
        \bottomrule
    \end{tabular}}
    \caption{Table of symbols and notations.}
    \label{tab:symbols}
\end{table}

\section{The function space and its Zariski closure}
    \label{sec:zariskiClosure}

From now on, to simplify notation, we   omit the index of the function space $\mathcal{M} = \mathcal{M}_{\bm{k},\bm{s}}$ and its parameterization map $\mu = \mu_{\bm{k},\bm{s}}$ when the LCN architecture $(\bm{k},\bm{s})$ is clear from context.
As the function space $\mathcal M$ is always  closed under multiplication by scalars (in other words, it forms a cone in $\RR^k$), it is natural to consider its projectivization, denoted by $\PP(\mathcal{M})$. This means that end-to-end filters that differ only by scalar multiplication are treated as equivalent. Similarly, we can projectivize the ambient space of the filters in each layer and replace the parameterization map $\mu$ with a morphism $\nu$ that composes filters up to scaling:
$$ \nu = \nu_{\bm{k},\bm{s}}: \PP_\mathbb{R}^{k_1-1} \times \cdots \times  \PP_\mathbb{R}^{k_L-1} \to \PP(\mathcal{M}_{\bm{k},\bm{s}}) \subseteq  \PP_\mathbb{R}^{k-1}.$$
\rev{\begin{remark}\label{rmk:segre}
    The map $\nu$ is the composition of a Segre embedding of $\PP_\mathbb{R}^{k_1-1} \times \cdots \times  \PP_\mathbb{R}^{k_L-1}$ followed by a linear projection.
    Moreover, $\nu$ itself is  a Segre embedding of $\PP_\mathbb{R}^{k_1-1} \times \cdots \times  \PP_\mathbb{R}^{k_L-1}$ (i.e., the linear projection is not required) if and only if 
    $S_l > \sum_{i=1}^{l-1}  (k_i-1) S_i$ for every layer $l$.
    Note that the latter condition appears in the last statement of Theorem \ref{thm:sing}.
    That statement can now be reinterpreted as follows:
    The projectivization of the function space is smooth if and only if it is a Segre variety (and not a proper linear projection from a Segre variety). 
\end{remark}}

%This 
\rev{The} 
projective setting has several technical advantages, including the fact that the image of a projective morphism is Zariski closed over the complex numbers~\cite[II,§4, Theorem 4.9]{hartshorne}.
Moreover, the morphism $\nu$ has finite fibers (see Remark \ref{rmk:finite-morphism-nu}), 
which enables us to use techniques from birational geometry.

In order to leverage these advantages, we introduce the \emph{complex function space} $\mathcal{M}^{\mathbb{C}}=\mathcal{M}^{\mathbb{C}}_{\boldsymbol{k},\boldsymbol{s}}$, consisting of complex filters that can be factorized according to the network architecture with complex filters in each layer. We also define the complex version of the projective morphism $\nu$:
$$ 
\nu^{\mathbb{C}} = \nu_{\boldsymbol{k},\boldsymbol{s}}^{\mathbb{C}}: \PP_\mathbb{C}^{k_1-1} \times \cdots \times  \PP_\mathbb{C}^{k_L-1} \to \PP(\mathcal{M}_{\boldsymbol{k},\boldsymbol{s}}^{\mathbb{C}}) \subseteq  \PP_\mathbb{C}^{k-1}.$$
It now follows from basic real algebraic geometry that the real Zariski closure $\overline{\mathcal{M}}$ of the real function space $\mathcal{M}$ is the set of real points in $\mathcal{M}^{\mathbb{C}}$. We provide a formal proof of this fact and summarize these observations using the identification of filters with polynomials.

\begin{proposition} If $\pi_1$ is the map defined in \eqref{eq:polynomials} and $\pi_1^\mathbb{C}$ is its complex counterpart, then:
\label{prop:functionspace-poly}
    \begin{align*}
          \pi_1(\mathcal{M}_{\boldsymbol{k},\boldsymbol{s}}) &= \left\lbrace P \in \mathbb{R}[x,y]_{k-1} \quad : \quad P =  P_L \cdots P_1, \quad P_i \in \mathbb{R}[x^{S_i},y^{S_i}]_{k_i-1} \right\rbrace, \\
                  \pi_1(\overline{\mathcal{M}}_{\boldsymbol{k},\boldsymbol{s}}) &= \left\lbrace P \in \mathbb{R}[x,y]_{k-1} \quad : \quad P =  P_L \cdots P_1, \quad P_i \in \mathbb{C}[x^{S_i},y^{S_i}]_{k_i-1} \right\rbrace,\\
         \pi_1^{\mathbb{C}}({\mathcal{M}}^\mathbb{C}_{\boldsymbol{k},\boldsymbol{s}}) =\pi_1^{\mathbb{C}}( \overline{\mathcal{M}}^\mathbb{C}_{\boldsymbol{k},\boldsymbol{s}}) &= \left\lbrace P \in \mathbb{C}[x,y]_{k-1} \quad : \quad P =  P_L \cdots P_1, \quad P_i \in \mathbb{C}[x^{S_i},y^{S_i}]_{k_i-1} \right\rbrace.
    \end{align*}    
\end{proposition}

\begin{proof}
It was already observed in \cite[Remark 2.8]{LCN} that an end-to-end filter $w = \mu(w_1, \ldots, w_L)$ of an LCN with strides $\bm{s}=(s_1,\ldots,s_L)$ is the coefficient vector of a polynomial with a sparse factorization as follows:
\begin{align*}
    \pi_1(w) = \pi_{S_L}(w_L) \cdot \pi_{S_{L-1}}(w_{L-1}) \cdots \pi_{S_2}(w_2) \cdot \pi_{S_1}(w_1),
\end{align*}
where $S_i = s_{i-1}\cdots s_1$ for $i>1$ and $S_1 = 1$.
That observation immediately implies the first equality of this proposition, which was already stated as Proposition \ref{prop:poly-mult} in Section~\ref{sec:main-results}. 

The equalities in the last row of the statement are the analog over the complex numbers of the first row, using that the image $\PP(\mathcal{M}^\CC)$ of the map $\nu^\mathbb{C}$ is Zariski closed.
Finally, the last equality implies the claim in the middle row.
This follows from the fact that for any subset $X \subseteq \mathbb{R}^n$,  its Zariski closures $\overline{X}^\mathbb{R}$  inside $\mathbb{R}^n$ and  $\overline{X}^\mathbb{C}$ inside $\mathbb{C}^n$ satisfy $\overline{X}^\mathbb{R} = \overline{X}^\mathbb{C}\cap\mathbb{R}^n$.

That fact can be been seen as follows. The inclusion of vanishing ideals $I_{\mathbb{R}}(X) \subseteq I_{\mathbb{C}}(X)$ implies the reverse inclusion of zero loci $\overline{X}^\mathbb{R} \supseteq \overline{X}^\mathbb{C}\cap\mathbb{R}^n$.
Moreover, for any polynomial $g \in I_{\mathbb{C}}(X)$ that vanishes on $X$, the real polynomials $\operatorname{Real}(g) = \frac{g+\overline{g}}{2}$ and $\operatorname{Imag}(g) = \frac{g-\overline{g}}{2i}$ also vanish on $X$, 
meaning that they also vanish on $\overline{X}^\mathbb{R}$ 
and so does $g = \operatorname{Real}(g) + i \operatorname{Imag}(g)$.
In other words, we have $I_{\mathbb{C}}(X) \subseteq I_{\mathbb{C}}(\overline{X}^\mathbb{R})$, 
which yields  $\overline{X}^\mathbb{R} \subseteq \overline{X}^\mathbb{C}$. 

Now, if $X:= \mathcal{M}$, then $\overline{X}^{\mathbb{R}} = \overline{\mathcal{M}}$ and $\overline{X}^{\mathbb{C}} = {\mathcal{M}}^{\mathbb{C}}$. To see the latter equality, note that $ \mathcal{M}^{\mathbb{C}}$ contains $X$ and since it is Zariski closed, we have $\overline{X}^{\mathbb{C}} \subseteq {\mathcal{M}}^{\mathbb{C}}$. To show the other inclusion, we consider a polynomial  $f \in I_{\mathbb{C}}(X)$ in the complex vanishing ideal of $X$. Then, $f \circ \mu^{\mathbb{C}}$ vanishes on all real inputs, thus $f \circ \mu^{\mathbb{C}} = 0$. This implies $f \in I_{\mathbb{C}}(\mathcal{M}^{\mathbb{C}})$, and therefore, $I(\overline{X}^{\mathbb{C}}) \subseteq I(\mathcal{M}^{\mathbb{C}})$. The
inclusion of ideals yields the reverse inclusion of varieties,   ${\mathcal{M}}^{\mathbb{C}} \subseteq \overline{X}^{\mathbb{C}}$, which finishes the proof.
\end{proof}

\begin{remark}
\label{rmk:finite-morphism-nu}
One direct consequence of Proposition \ref{prop:functionspace-poly} is that each fiber $\nu^{-1}(w)$ or $(\nu^{\mathbb{C}})^{-1}(w)$  is finite, based on the different arrangements of the roots of the polynomial $P=\pi_1^{\mathbb{C}}(w)$ into factors $P=P_L \cdots P_1$ according to the LCN architecture.
\end{remark}

Hence, using the projective morphism $\nu$, we can easily compute the dimension of LCN function spaces and show that they are closed in the Euclidean topology.

\begin{proof}[Proof of Theorem \ref{thm:dim-mu}]
Since the function space $\mathcal{M}$ has a polynomial parameterization, it is semialgebraic by Tarski-Seidenberg.
To see that $\mathcal{M}$ is Euclidean closed, we consider the projective spaces appearing in the map $\nu$ endowed with the quotient topology of the Euclidean topology on their underlying real vector spaces.
Then these spaces are Hausdorff (unlike in the Zariski topology) and compact.
Hence, since $\nu$ is continuous, its image is closed.
Since the function space $\mathcal{M}$ is the affine cone over the image of $\nu$, it is closed in the Euclidean topology. To find the dimension of the function space, we use again the projective morphism~$\nu$. By Remark~\ref{rmk:finite-morphism-nu}, every fiber is zero-dimensional and hence its domain and image have the same dimension by {\cite[II,§3, Exercise 3.22]{hartshorne}}. We conclude that
$$
\dim \mathcal{M}_{\bm{k},\bm{s}}-1  =\dim(\PP(\mathcal{M}_{\bm{k},\bm{s}}))=\dim (\PP^{k_1-1} \times \cdots \times  \PP^{k_L-1}) = k_1+\cdots+k_L- L.
$$
~
\end{proof} 

Our next goal is to prove Theorem~\ref{thm:thick-closed}. For that, we study the root structure of the polynomial factors in Proposition~\ref{prop:functionspace-poly}. 

\begin{definition}
    Given a positive integer $s$, an \emph{$s$-hyperroot} is any binomial of the form $ax^s+by^s$ with $a,b \in \CC$, $(a,b) \neq (0,0)$.
    We say that the $s$-hyperroot is \emph{real} if $a,b \in \RR$.
    An $s$-hyperroot $R$ is \emph{non-real} if $\alpha R$ is not real for any $\alpha \in \CC\setminus\{0\}$.
\end{definition}

\begin{lemma}\label{lem:hyperrootPairs}
    If a non-real $s$-hyperroot $R$ divides a real polynomial $P \in \RR[x,y]$, then its complex conjugate $\overline{R}$ divides $\frac{P}{R}$.
\end{lemma}

\begin{proof}
    The statement is clear for $s=1$.
    For $s>1$, we write $R = L_1 \cdots L_s$, where $L_i \in \CC[x,y]_1$. Since the linear factors $L_i$ correspond to the roots of $R$ which was assumed to be non-real, each $L_i$ is non-real and its the complex conjugate $\overline{L_i}$ does not divide $R$.
    However, $\overline{L_i}$ must divide the real polynomial $P$. 
    Hence, $\overline{R} = \overline{L_1} \cdots \overline{L_s}$ divides $\frac{P}{R}$.
\end{proof}

With the notion of hyperroots, we now prove equation \eqref{eq:reduction}.

\begin{lemma}
\label{lem:functionSpaceReduction}
    Let $(\bm{k},\bm{s})$ be an LCN architecture and let $(\tilde{\bm{k}},\tilde{\bm{s}})$ be its associated reduced architecture. Then, 
    $\mathcal{M}_{{\bm k},{\bm s}} \subseteq \mathcal{M}_{\tilde{{\bm k}},\tilde{{\bm s}}}$  and  $\overline{\mathcal{M}}_{{\bm k},{\bm s}} = \overline{\mathcal{M}}_{\tilde{{\bm k}},\tilde{{\bm s}}}$.
    \end{lemma}

\begin{proof}
We start by showing the inclusion $\mathcal{M}_{\bm{k},\bm{s}} \subseteq \mathcal{M}_{\tilde{\bm{k}},\tilde{\bm{s}}}$. By Proposition \ref{prop:functionspace-poly}, every filter in the function space $\mathcal{M}_{\bm{k},\bm{s}}$ corresponds to a polynomial that can be factorized as $P=P_L\cdots P_1$, where $P_i \in \mathbb{R}[x^{S_i},y^{S_i}]_{k_i-1}$. Using the notation in Definition \ref{def:reduction}, setting $\tilde{P}_{j+1}:= P_{l_{j+1}}\cdots P_{l_j+1}$ yields a factorization $P=\tilde{P}_M \cdots \tilde{P}_1$ according to the reduced architecture, i.e., $\pi_1^{-1}(P) \in \mathcal{M}_{\tilde{\bm{k}},\tilde{\bm{s}}}$.
To prove the equality of Zariski closures, it suffices to show that $\overline{ \mathcal{M}}_{\tilde{\bm{k}},\tilde{\bm{s}}}$ is a subset of $\overline{ \mathcal{M}}_{\bm{k},\bm{s}}$. By Proposition \ref{prop:functionspace-poly}, every filter in $\overline{ \mathcal{M}}_{\tilde{\bm{k}},\tilde{\bm{s}}}$ corresponds to a real polynomial $P$ with a factorization $P=\tilde{P}_M \cdots \tilde{P}_1$, where $\tilde{P}_i \in \mathbb{C}[x^{\tilde{S}_i},y^{\tilde{S}_i}]_{\tilde{k}_i+1}$. Since every complex factor $\tilde{P}_i$ can be written as a product of $\tilde{S}_i$-hyperroots, we can find a complex factorization of $\tilde{P}_i$ according to the stride-one architecture $(\tilde {\bm k}^i, \bm 1)$. This yields a complex factorization of the real polynomial $P$ according to the original architecture $(\bm{k},\bm s)$, so $\pi_1^{-1}(P) \in  \overline{\mathcal{M}}_{\bm{k},\bm s}$ by Proposition \ref{prop:functionspace-poly}.  
 \end{proof}

 The final ingredient for our proof of Theorem \ref{thm:thick-closed} is to show that the zero filter is a singular point for all non-stride-one architectures.

\begin{lemma}
    \label{lem:zeroSing}
    If not all strides are equal to one,  then the
    algebraic degree of $\overline{\mathcal{M}}$ is larger than one.
    In particular, the zero filter is a singular point of the affine cone $\overline{\mathcal{M}}$. 
\end{lemma}

\begin{proof}
It is sufficient to show that there are two filters $w, w' \in \overline{\mathcal{M}}$ such that their sum is not contained in $\overline{\mathcal{M}}$.
We do this by constructing their corresponding polynomials $P = \pi_1(w)$ and $P' = \pi_1(w')$.
Let $l$ be the minimal layer such that $s_l>1$. 
Note that this implies that $S_{l+1}=s_l$ and $S_j = 1$ for all $j \leq l$.
We pick polynomials $P_i \in \RR[x^{S_i},y^{S_i}]_{k_i-1}$ for $i>l+1$ arbitrarily.
We choose an arbitrary $Q_{l+1} \in \RR[x^{s_{l}},y^{s_{l}}]_{k_{l+1}-2}$
and set $P_{l+1} := x^{s_l}Q_{l+1}$.
Finally, we pick $Q_l \in \RR[x,y]_{k_l-2}$, set $P_l := yQ_l$, and choose $P_j \in \RR[x,y]_{k_i-1}$ for $i<l$ such that the product $P_l \cdots P_1$ is not divisible by any $s_l$-hyperroot.
Then, $P:= P_L\cdots P_1 \in \pi_1(\mathcal{M})$.
The second polynomial $P' := P'_L \cdots P'_1 \in \pi_1(\mathcal{M})$ is constructed by setting $P'_i := P_i$ for $i > l+1$, $P'_{l+1} := y^{s_l}Q_{l+1}$, $P'_l := x Q_l$, and $P'_j := P_j$ for $j < l$.
Now we have that
$P+P' = P_L \cdots P_{l+2}Q_{l+1}Q_lP_{l-1}\cdots P_1 (x^{s_l}y + xy^{s_l})$.
Since $x^{s_l}y + xy^{s_l}$ is not divisible by any $s_l$-hyperroot, the sum $P+P'$ does not contain enough $s_l$-hyperroots to admit a factorization according to the architecture $(\bm k, \bm s)$, i.e., $P+P' \notin \pi_1(\overline{\mathcal{M}}_{\bm k, \bm s})$ by Proposition \ref{prop:functionspace-poly}.
\end{proof}

\begin{proof}[Proof of Theorem \ref{thm:thick-closed}]
By  Theorem \ref{thm:dim-mu}, the function space $\FS$ is thick if and only if
\begin{align}
\label{eq:thickEquality}
    k_1 + \sum_{l=2}^L (k_l-1) = \dim (\FS) = k = k_1 + \sum_{l=2}^L (k_l-1) S_l.
\end{align}
Since $S_l \geq 1$ and we always assume that $k_l>1$, 
the equality in \eqref{eq:thickEquality} holds if and only if $S_2 = \cdots = S_L = 1$. 
The latter is equivalent to $s_1 = \ldots = s_{l-1}=1$.
Since we always assume the last stride $s_L$  to be one, we have proven assertion \textit{a)}.

For assertion \textit{b)}, we observe that for thick LCN function spaces, their Zariski closures are vector spaces and thus smooth. 
Thin LCN function spaces have at least one stride larger one by assertion \textit{a)} and  are thus singular by Lemma \ref{lem:zeroSing}.

We now prove assertion \textit{c)}.
Part \textit{c2)} was shown in \cite[Theorem 4.1]{LCN}.
For parts \textit{c1)} and \textit{c3)}, we observe that the function space $\FS$ is Zariski closed if and only if every real polynomial $P$ with a complex factorization $P =  P_L \cdots P_1$,  $P_i \in \mathbb{C}[x^{S_i},y^{S_i}]_{k_i-1}$ admits such a factorization with real factors 
(this follows from Proposition \ref{prop:functionspace-poly}).

We start by proving one direction of part \textit{c1)}, and assume that every layer $l \in \{2, \ldots, L \}$ satisfies that the filter size $k_l$ is odd or $S_l > \sum_{i=1}^{l-1}(k_i-1) S_i$.
We show by induction on $L$ that every real polynomial $P =  P_L \cdots P_1$ with $P_i \in \mathbb{C}[x^{S_i},y^{S_i}]_{k_i-1}$ admits the analogous real factorization. 
The base case of the induction with $L=1$ is trivial.
For $L>1$, we consider the $S_L$-hyperroots dividing  $P$.
We observe that all roots of the factor $P_L$ correspond to $S_L$-hyperroots of $P$, but that $P$ might also have other $S_L$-hyperroots.
By Lemma \ref{lem:hyperrootPairs}, all its non-real $S_L$-hyperroots appear in complex conjugated pairs.
We now distinguish two cases.
First, if $k_L$ is odd, i.e., the degree of $P_L$ is even, then we can rearrange the $S_L$-hyperroots of $P$ into a new factorization $P = \tilde{P}_L \cdots \tilde{P}_1$ of the same format such that the new factor $\tilde{P}_L$ becomes real. 
Second, the inequality  $S_L> \sum_{i=1}^{L-1}  (k_i-1)S_i$ means that no $S_L$-hyperroot fits into the polynomial $P_{L-1}\cdots P_1$ since $\deg(P_{L-1}\cdots P_1)=\sum_{i=1}^{L-1}  (k_i-1)S_i$ is too small. Hence, $P_L$ is the product of all $S_L$-hyperroots of $P$ and thus must be real.
In either case, we have found a factorization of $P = \tilde{P}_L \cdots \tilde{P}_1$ according to the LCN architecture with $\tilde{P}_L$ real. 
So $\tilde{P}_{L-1} \cdots \tilde{P}_1$ is real and we can apply the induction hypothesis to the $(L-1)$-layer LCN that omits the last layer. 

For the converse direction of \textit{c1)}, we assume that the architecture $(\bm{k},\bm{s})$ is reduced and that there exists a layer $l \in \{2, \ldots, L \}$ with even filter size $k_l$ and stride relation $S_l\le \sum_{i=1}^l (k_i-1)S_i$.
We fix  $P_l := (x^{S_l}+2iy^{S_l})(x^{S_l}-2iy^{S_l}) \cdot (x^{S_l}+4iy^{S_l})(x^{S_l}-4iy^{S_l}) \cdots (x^{S_l}+k_liy^{S_l})$ of odd degree $k_l-1$ and $R :=x^{S_l}-k_liy^{S_l}$.
Note that $P_l\cdot R$ has no $S$-hyperroot for any $S > S_l$. 
As $S_l\le \sum_{i=1}^l (k_i-1)S_i$,  we can choose $P_i \in \CC[x^{S_i},y^{S_i}]_{k_i-1}$ for $1 \leq i < l$ such that $R$ divides $P_{l-1} \cdots P_1$.
Moreover, we may choose the $P_i$ such that the quotient $\frac{P_{l-1} \cdots P_1}{R}$ is real and has no $S_l$-hyperroot (since $S_l > S_{l-1}$).
Now, we choose the factors $P_j \in \RR[x^{S_j},y^{S_j}]_{k_j-1}$ for $l < j \leq L$ arbitrarily.
The resulting polynomial $P := P_L \cdots P_1$ is real and hence $\pi_1^{-1}(P) \in \overline{\mathcal{M}}_{\bm{k},\bm{s}}$.
If $P = Q_L \cdots Q_1$ is any other complex factorization of $P$ according to the \emph{reduced} architecture, then $P_L \cdots P_{l+1} = Q_L \cdots Q_{l+1}$ (up to scaling).
Thus,  $Q_l$ is the product of $S_l$-hyperroots of $\frac{P}{P_L \cdots P_{l+1}}$, i.e., $Q_l$  divides $P_l \cdot R$. 
In particular, the odd-degree factor $Q_l$ cannot be real and so $\pi_1^{-1}(P) \notin \FS$.

Finally, we prove \textit{c3)}. We use the notation in Definition \ref{def:reduction}, and start by assuming that the function spaces $\mathcal{M}_{\tilde{\bm{k}},\tilde{\bm{s}}}$ of the associated reduced architecture and $\mathcal{M}_{\tilde{\bm{k}}^j,{\bm{1}}}$ (for $1 \leq j \leq M$) of the associated stride-one architectures are Zariski closed.
Let $P = P_L \cdots P_1$ be a real polynomial with $P_i \in \CC[x^{S_i},y^{S_i}]_{k_i-1}$.
Then $\tilde{P}_{j+1} := P_{l_{j+1}} \cdots P_{l_j+1}$ yields a complex factorization $P = \tilde{P}_M \cdots \tilde{P}_1$ according to the reduced architecture. 
Since $\mathcal{M}_{\tilde{\bm{k}},\tilde{\bm{s}}}$ is Zariski closed, we can find a real factorization 
$P = \tilde{Q}_M \cdots \tilde{Q}_1$ of the same format.
Moreover, since $\mathcal{M}_{\tilde{\bm{k}}^{j+1},{\bm{1}}}$ is Zariski closed, every real polynomial of degree $\tilde{k}_{j+1}-1$ can be factorized  into real factors according to the stride-one architecture $(\tilde{\bm{k}}^{j+1},{\bm{1}})$.
In particular, we can factorize $\tilde{Q}_{j+1} = Q_{l_{j+1}} \cdots Q_{l_j+1}$ such that $Q_i \in \RR[x^{S_i},y^{S_i}]_{k_i-1}$.

For the converse direction of \textit{c3)}, we assume that the function space $\FS$ is Zariski closed.
We see directly from Lemma \ref{lem:functionSpaceReduction} that the function space of the associated reduced architecture must be Zariski closed as well.
We assume for contradiction that the function space $\mathcal{M}_{\tilde{\bm{k}}^j,{\bm{1}}}$ of one of the associated stride-one architectures is not Zariski closed. 
By \textit{c2)}, this means that at least two of its filter sizes are even.
In particular, any  polynomial $\tilde{P}_j \in \RR[x,y]_{\tilde{k}_j-1}$ with one or zero real roots (depending on the parity of its degree) cannot be factorized according to the stride-one architecture $(\tilde{\bm{k}}^{j},{\bm{1}})$.
We now fix such a polynomial $\tilde{P}_j$ such that it has no $s$-hyperroot for any $s>1$.
We further pick $\tilde{P}_i \in \RR[x,y]_{\tilde{k}_i-1}$ for $j < i \leq M$ arbitrarily and for $1 \leq i < j$ such that $\tilde{P}_i$ has no $s$-hyperroot for any $s>1$.
Then $P := \tilde{P}_M(x^{\tilde{S}_M},y^{\tilde{S}_M}) \cdots \tilde{P}_1(x,y)$ is a factorization of the real polynomial $P$ according to the reduced architecture $(\tilde{\bm{k}},\tilde{\bm{s}})$.
Since $\FS$ and $\mathcal{M}_{\tilde{\bm{k}},\tilde{\bm{s}}}$ are Zariski closed, there is a factorization $P = \tilde{Q}_M \cdots \tilde Q_1$ into real factors $\tilde{Q}_i$ according to the reduced architecture such that each $\tilde{Q}_i$ factorizes according to its associated stride-one architecture. 
However, by our construction, $\tilde{Q}_j$ must be equal to $\tilde{P}_j$ (up to scaling), which contradicts that $\tilde{P}_j$ cannot be factorized according to $(\tilde{\bm{k}}^{j},{\bm{1}})$.
\end{proof}

\section{Critical points of the parametrization}
\label{sec:criticalMu}

When determining the critical points of the LCN parametrization map $\mu$, it is once again easier to work with the projective morphism $\nu$.
This is because the kernel of the differential of $\nu$ at a regular point is zero (due to the finiteness of the fibers of $\nu$).
Writing $\bar x$ for the equivalence class of $x \in \RR^{n+1}\setminus \{0\}$ in $\mathbb{P}^n_{\mathbb{R}}$,
the tangent space $T_{\bar{x}}\mathbb{P}^n_{\mathbb{R}}$ is canonically isomorphic to
$\mathrm{Hom}_{\RR}(\langle x \rangle, \mathbb{R}^{n+1}/\langle x \rangle)$ \cite[Example 6.24]{shaf2}.
In our calculations below, we view that tangent space as $\mathbb{R}^{n+1}/\langle x \rangle \;(\cong \mathrm{Hom}_{\RR}(\langle x \rangle, \mathbb{R}^{n+1}/\langle x \rangle))$ for simpler notation.
This still captures the relevant geometry by modding out the trivial kernel vectors of the differentials of the affine parametrization map $\mu$, as we see below in Lemma \ref{lem:projectiveCrit} and its proof.
The underlying reason is that for the map $f: \mathbb{R}^{n+1}\setminus \{0\} \to \mathbb{P}^n_{\mathbb{R}}$ that sends $x$ to $\bar{x}$, the differential  $d_xf: \mathbb{R}^{n+1} \to T_{\bar{x}}\mathbb{P}^{n}_{\mathbb{R}}$ is surjective and its kernel is the line $\langle x \rangle$ spanned by $x$.

\begin{lemma} \label{lem:projectiveCrit}
    A filter tuple $\theta=(w_1,\ldots,w_L) \in \mathbb{R}^{k_1} \times \cdots \times \mathbb{R}^{k_L}$ with $w_i\neq 0$ (for all $i$) is a critical point of  $\mu$ if and only if its corresponding projective point in $\PP^{k_1-1}\times \cdots \times\PP^{k_L-1}$ is a critical point of $\nu$.
\end{lemma}

\begin{proof}
The maps $\mu$ and $\nu$ and their differentials form the following commutative diagrams:
\begin{center}
    \begin{tikzcd}
\prod_{i=1}^{L}(\mathbb{R}^{k_i} \setminus \{0\}) \arrow{r}{\mu} \arrow{d} & \RR^{k}\setminus \{0\} \arrow{d} \\   
\prod_{i=1}^{L}\mathbb{P}^{k_i-1} \arrow[swap]{r}{\nu} & \mathbb{P}^{k-1} 
\end{tikzcd}  \quad
\begin{tikzcd}
\prod_{i=1}^{L} \mathbb{R}^{k_i}  \arrow{r}{d_\theta\mu} \arrow{d} & \RR^{k} \arrow{d} \\   
\prod_{i=1}^{L} (\mathbb{R}^{k_i} / \langle w_i \rangle) \arrow[swap]{r}{d_{\bar\theta}\nu} & \mathbb{R}^{k} / \langle \mu(\theta) \rangle
\end{tikzcd}
\end{center}
In particular, we have that $\mathrm{im}\, d_{\bar\theta}\nu \cong (\mathrm{im}\, d_{\theta}\mu) / \langle \mu(\theta) \rangle$, which implies the assertion. 
\end{proof}

\begin{lemma}
\label{lem:zeroFilterCrit}
Let $L>1$.
    Every filter tuple $\theta=(w_1,\ldots,w_L) \in \mathbb{R}^{k_1} \times \cdots \times \mathbb{R}^{k_L}$ where one of the filters $w_i$ equals zero is a critical point of  $\mu$. 
\end{lemma}

\begin{proof}
Using our identification $\pi_1$ of filters with polynomials, the  map $\mu$ becomes
\begin{align} \label{eq:muPol}
    \mu: \mathbb{R}[x^{S_1},y^{S_1}]_{k_1-1} \times \cdots\times \mathbb{R}[x^{S_L},y^{S_L}]_{k_L-1} \to \mathbb{R}[x,y]_{k-1}, 
    \quad (P_1, \ldots, P_L) \mapsto P_L \cdots P_1, 
\end{align}
and we can write the differential at $\theta=(P_1, \ldots, P_L)$ as
\begin{align*}
    d_{\theta} \mu: \mathbb{R}[x^{S_1},y^{S_1}]_{k_1-1} \times \cdots\times \mathbb{R}[x^{S_L},y^{S_L}]_{k_L-1} &\to \mathbb{R}[x,y]_{k-1}, \\
     (\dot P_1, \ldots, \dot P_L) &\mapsto \dot P_L \cdot P_{L-1} \cdots P_1 + \cdots +  P_L \cdots P_2 \cdot \dot P_1.
\end{align*}
If one of the polynomials, say $P_1$, is zero, the differential simplifies to
$d_{\theta} \mu: (\dot P_1, \ldots, \dot P_L) \mapsto   P_L \cdots P_2 \cdot \dot P_1$, and so we obtain $\rank(d_\theta \mu) \leq k_1$.
Since we assume all filter sizes to be larger than one and $L \geq 2$, we conclude $\rank(d_\theta \mu) \leq k_1 < k_1 + (k_2-1) + \cdots + (k_L-1) = \dim \mathcal{M}$. 
\end{proof}

\begin{remark} \label{rem:specialArchCrit}
Lemma \ref{lem:zeroFilterCrit} does not apply to  single-layer LCN architectures. 
If $L=1$, the parametrization map $\mu$ is the identity map and thus smooth.
\end{remark}

We now compute all critical points of $\mu$ for two-layer architectures.

\begin{proposition} \label{prop:twoLayerCritMu}
Consider a two-layer LCN architecture.
A filter pair $(w_1,w_2)$ is a critical point of $\mu$ if and only if one of the filters is zero or the polynomials $\pi_{s_1}(w_2)$ and $\pi_1(w_1)$ have a common $s_1$-hyperroot.
\end{proposition}

\begin{proof}
By Lemma~\ref{lem:zeroFilterCrit}, we can assume that neither $w_1$ nor $w_2$ are zero. 
Then, by Lemma~\ref{lem:projectiveCrit}, we can use the projective map $\nu$ to determine whether $(w_1,w_2)$ is a critical point. 
Using our identification with polynomials $P_1 = \pi_1(w_1)$ and $P_2 = \pi_{s_1}(w_2)$, we can write 
$\nu: \PP (\mathbb{R}[x,y]_{k_1-1}) \times \PP( \mathbb{R}[x^{s_1},y^{s_1}]_{k_2-1}) \to \PP(\mathbb{R}[x,y]_{k-1}), (P_1,P_2) \mapsto P_2 \cdot P_1$ and 
\begin{align*}
    d_{(P_1,P_2)} \nu: \mathbb{R}[x,y]_{k_1-1} / \langle P_1 \rangle  \times \mathbb{R}[x^{s_1},y^{s_1}]_{k_2-1} / \langle P_2 \rangle &\to \mathbb{R}[x,y]_{k-1} / \langle P_1P_2 \rangle, \\ ([\dot P_1], [\dot P_2]) &\mapsto [\dot P_2 P_1 + P_2 \dot P_1].
\end{align*}
Since $\nu$ has finite fibers (see Remark \ref{rmk:finite-morphism-nu}), the kernel of its differential at a regular point is trivial (i.e., zero-dimensional).
At a critical point $(P_1,P_2)$, there is a non-trivial kernel element.
We start by showing that each critical point $(P_1,P_2)$ satisfies that the polynomials $P_1$ and $P_2$ have a common $s_1$-hyperroot.
A non-trivial kernel element of the differential means that there is $([\dot P_1], [\dot P_2]) \neq ([ P_1], [ P_2])$ with $[\dot P_2 P_1 + P_2 \dot P_1]=[P_1P_2]$.
In particular, $P_2$ divides $\dot P_2 P_1$.
Any factor of $P_2$ that is linear in $x,y$ is part of a $s_1$-hyperroot of $P_2$.
Moreover, if such a linear factor divides $\dot P_2$, its whole $s_1$-hyperroot must divide $\dot P_2$.
Hence, every $s_1$-hyperroot of $P_2$ must divide either $\dot P_2$ or $P_1$.
If such a $s_1$-hyperroot divides $P_1$, we are done.
Otherwise, $P_2$ and $\dot P_2$ are equal up to scaling, 
and $[P_1P_2] = [\dot P_2 P_1 + P_2 \dot P_1] = [P_2 \dot P_1]$ implies that $P_1$ and $\dot P_1$ are equal up to scaling as well, which contradicts that we started from a non-trivial kernel element.
For the other direction, if $P_1$ and $P_2$ have a common $s_1$-hyperroot $r \in \RR[x^{s_1},y^{s_1}]_1$, then we can write $P_2 = r Q_2$ and $P_1 = r Q_1$.
We now pick $f \in \RR[x^{s_1},y^{s_1}]_1$ such that $\gcd(P_1P_2,f)=1$.
Then, $([f Q_1],[-f Q_2])$ is a non-trivial kernel element of the differential $d_{(P_1,P_2)} \nu$, which shows that $(P_1,P_2)$ is a critical point of $\nu$.
\end{proof}

To determine the critical points of the parametrization map $\mu$ for arbitrarily many layers, we start by proving a technical lemma, which provides a partial understanding of the image of the differential of $\mu$.

\begin{lemma}
    \label{lem:imageDiffMu}
    Let  $\theta=(w_1,\ldots,w_L) \in \RR^{k_1} \times \cdots \times \RR^{k_L}$ be such that $w_i \neq 0$ for all $i$, 
    the polynomials $\pi_{S_1}(w_1), \ldots, \pi_{S_L}(w_L)$ are pairwise coprime, 
    and their product $\pi_1(\mu(\theta))\in\RR[x,y]_{k-1}$ is a polynomial in $x^{s},y^{s}$ 
    (i.e., $\pi_1(\mu(\theta)) \in \RR[x^s,y^s]_{k'-1}$, where $s = s_1 \cdots s_L$ is the product of the strides and $(k'-1)s=k-1$).
    Then we have that $\RR[x^s,y^s]_{k'-1} \subseteq \pi_1(\mathrm{im}(d_\theta \mu))$. 
\end{lemma}

\begin{proof}
We omit writing $\pi_1$. Instead, we directly view $\mu$ as the polynomial multiplication map in \eqref{eq:muPol} and work with the polynomials $P_i := \pi_{S_i}(w_i)$.
    We prove the assertion by induction on $L$.
    For a single layer (i.e., $L=1$), the map $\mu$ is the identity and the assertion is trivial.
    For the induction step, we consider the $(L-1)$-layer LCN architecture that omits the $L$-th layer and denote its parametrization map by $\tilde \mu$.
    Now the  tuple $\tilde \theta = (P_1, \ldots, P_{L-1})$ satisfies the assumptions in Lemma \ref{lem:imageDiffMu} since $\tilde \mu(\tilde \theta) = \frac{\mu(\theta)}{P_L}$ is a polynomial in $x^{S_L},y^{S_L}$.
    Writing $\tilde k-1$ for the degree of $\tilde{\mu}(\tilde \theta)$ and $(\tilde k'-1)S_L=\tilde k-1$, we have $\tilde \mu(\tilde \theta) \in \RR[x^{S_L},y^{S_L}]_{\tilde k'-1}$.
    Applying the induction hypothesis yields that \begin{align}
        \label{eq:ImageIH}
        \RR[x^{S_L},y^{S_L}]_{\tilde k'-1} \subseteq \mathrm{im}(d_{\tilde \theta} \tilde\mu).
    \end{align}

    We now perform a change of variables $\tilde x = x^{S_L},\tilde y = y^{S_L}$ and consider the map $\varphi: \RR[\tilde x, \tilde y]_{k_L-1} \times  \RR[\tilde x, \tilde y]_{\tilde k'-1} \to \RR[\tilde x, \tilde y]_{ k_L + \tilde k'-2}$ that multiplies two polynomials. 
    The assumptions of this lemma in particular yield that $P_L$ and $\tilde \mu(\tilde \theta)$ are coprime. 
    Hence, by Proposition \ref{prop:twoLayerCritMu}, the pair $(P_L,\tilde \mu(\tilde \theta))$ is a regular point of the map $\varphi$.
    This means that the differential $d_{(P_L,\tilde \mu(\tilde \theta))} \varphi$ is surjective. 
    Thus, for any $P \in \RR[x^s,y^s]_{k'-1} = \RR[\tilde x, \tilde y]_{(k'-1)s_L} = \RR[\tilde x, \tilde y]_{ k_L + \tilde k'-2}$, 
    there are polynomials $\dot P_L \in \RR[\tilde x, \tilde y]_{k_L-1}$ and $\dot Q \in \RR[\tilde x, \tilde y]_{\tilde k'-1}$ such that $P = \dot P_L \cdot \tilde \mu(\tilde \theta) + P_L \cdot \dot Q$.
    Moreover, due to \eqref{eq:ImageIH}, there are $\dot P_i \in \RR[x^{S_i},y^{S_i}]_{k_i-1}$ (for $1 \leq i \leq L-1$) such that $\dot Q = d_{\tilde \theta} \tilde\mu (\dot P_1, \ldots, \dot P_{L-1})$.
    Therefore, $P = \dot P_L \cdot \tilde \mu(\tilde \theta) + P_L \cdot d_{\tilde \theta} \tilde\mu (\dot P_1, \ldots, \dot P_{L-1}) = d_\theta \mu (\dot P_1, \ldots, \dot P_L) \in \mathrm{im}(d_\theta \mu)$.
\end{proof}

This lemma enables us to show that the critical points of the parametrization map $\mu$ can be understood from the critical points of subnetworks with fewer layers. 
For a projective LCN parametrization map $\nu: \PP^{k_1-1} \times \ldots \times \PP^{k_L-1} \to \PP^{k-1}$, 
we denote by $\tilde \nu: \PP^{k_1-1} \times \ldots \times \PP^{k_{L-1}-1} \to \PP^{\tilde k-1}$ the parametrization map that is obtained by omitting the last layer. 
Moreover, we write $\varphi: \PP^{\tilde k-1} \times \PP^{k_L-1}  \to \PP^{k-1}$ for the two-layer LCN parametrization map that recovers $\nu$ from $\tilde \nu$, i.e.,  
$\nu = \varphi \circ (\tilde \nu \times \mathrm{id}_{\PP^{k_L-1}})$.

\begin{proposition} \label{prop:lessLayerCritMu}
Let $L>1$ and let $\theta=(w_1,\ldots,w_L) \in \mathbb{P}^{k_1-1} \times \ldots \times \mathbb{P}^{k_L-1}$. Then $\theta$ is a critical point of $\nu$ 
if and only if one of the following holds:
\begin{enumerate}
    \item[(a)] $\tilde \theta = (w_1,\ldots,w_{L-1})$ is a critical point of $\tilde \nu$ or 
    \item[(b)]  $(\tilde \nu(\tilde \theta), w_L)$ is a critical point of $\varphi$.
\end{enumerate}
\end{proposition}

\begin{proof}
We start by assuming that neither condition (a) nor (b) are satisfied, and show that  $\theta$ is  a regular point of $\nu$. 
The converse of condition (a) means that the vector space $V := \mathrm{im}\, d_{ \theta} (\tilde \nu \times \mathrm{id}_{\PP^{k_L-1}})$ has the expected dimension $\sum_{i=1}^L (k_i-1)= \dim \PP(\mathcal{M})$.
Since $\varphi$ has finite fibers (see Remark \ref{rmk:finite-morphism-nu}), the converse of condition (b) is that the differential $d_{(\tilde \nu(\tilde \theta), w_L)} \varphi$ is injective.
Together with the chain rule, the converses of (a) and (b) imply that
$\dim \mathrm{im}\, d_\theta\nu = \dim \mathrm{im}\, \left( (d_{(\tilde \nu(\tilde \theta), w_L)} \varphi)\!\mid_V \right) = \dim V = \dim \PP(\mathcal{M})$, which shows that $\theta$ is a regular point of $\nu$.
Next, we assume that condition (a) holds, i.e., $\dim V < \dim \PP(\mathcal{M})$.
Then the chain rule yields $\dim \mathrm{im}\, d_\theta\nu = \dim \mathrm{im}\, \left( (d_{(\tilde \nu(\tilde \theta), w_L)} \varphi)\!\mid_V \right) \leq \dim V < \dim \PP(\mathcal{M})$, so $\theta$ is a critical point of $\mu$.
Finally, we show that condition (b) implies that  $\theta$ is a critical point of $\nu$.
This is the technical part of the proof. 
As in Proposition \ref{prop:twoLayerCritMu}, we work directly with the polynomials $P_i := \pi_{S_i} (w_i)$ and view the maps $\nu$, $\tilde \nu$ and $\varphi$ as multiplying polynomials.
By Proposition \ref{prop:twoLayerCritMu}, condition (b) means that the polynomials $\tilde P := \tilde \nu(\tilde \theta) = P_1 \ldots P_{L-1}$ and $P_L$ have a common $S_L$-hyperroot $r \in \PP(\RR[x^{S_L},y^{S_L}]_1)$.
Now we distinguish two cases.
First, if $r \notin \lbrace x^{S_L},y^{S_L} \rbrace$, then the hyperroot factorizes as $r = r_1 \cdots r_{L-1}$ such that each factor $r_i \in \PP(\RR[x^{S_i},y^{S_i}])$ divides $P_i$ and
the factors $r_i$ are pairwise coprime (as polynomials in $\PP(\RR[x,y])$).
Hence, we can apply Lemma \ref{lem:imageDiffMu} to the map that multiplies the polynomials $r_1, \ldots, r_{L-1}$:
This yields that for every $f \in \RR[x^{S_L},y^{S_L}]_1$ there are $f_1, \ldots, f_{L-1}$ with $f_i \in \RR[x^{S_i},y^{S_i}]$ and $\deg(f_i)=\deg(r_i)$ such that 
$f_1 r_2 \cdots r_{L-1} + \cdots + r_1 \cdots r_{L-2}f_{L-1}=f$.
Thus, writing $P_L = r Q_L$ and $P_i = r_i Q_i$ for all $i <L$, we see that $([f_1Q_1], \ldots, [f_{L-1}Q_{L-1}], [-fQ_L])$ is in the kernel of the differential $d_{(P_1, \ldots, P_L)}\nu$.
Choosing $f \in \RR[x^{S_L},y^{S_L}]_1$ such that $\gcd(\tilde P P_L,f)=1$, ensures that that kernel element is non-trivial.
Since $\nu$ has finite fibers, the existence of a non-trivial kernel element shows that $\theta=(P_1, \ldots, P_L)$ is a critical point of $\nu$.
Second, if $r =x^{S_L}$ or $r = y^{S_L}$, we may assume the former without loss of generality. 
Since each $S_i$ divides $S_{i+1}$, there are two layers $i$ and $j$ with $1 \leq i < j \leq L$ such that $P_i$ and $P_j$ have $x^{S_j}$ as a common factor.
Writing $P_i = x^{S_j} Q_i$ and $P_j = x^{S_j} Q_j$, we construct a non-trivial kernel element $([\dot P_1], \ldots, [\dot P_L])$ of $d_\theta \nu$ by choosing $\dot P_i := y^{S_j} Q_i$, $\dot P_j := -y^{S_j} Q_j$, and $\dot P_m := 0$ for all other layers. 
\end{proof}

\begin{proof}[Proof of Theorem \ref{thm:crit-gene}]
    By Lemma \ref{lem:zeroFilterCrit}, we can assume that no filter $w_i$ is zero.
    Then, by Lemma \ref{lem:projectiveCrit}, we can use the projective map $\nu$ to determine whether the given filter tuple is a critical point. 
    Now Theorem \ref{thm:crit-gene} follows by induction on $L$ from Propositions \ref{prop:twoLayerCritMu} and \ref{prop:lessLayerCritMu}.
\end{proof}

\begin{remark}
\label{rmk:complex-proj-functionspace}
The description of critical points in Theorem \ref{thm:crit-gene} is the same when working over the complex numbers instead of the reals.
\end{remark}

\section{Singular points}
\label{sec:singular}

The goal of this section is to prove Theorem \ref{thm:sing}.
Hence, throughout this section, we consider a reduced LCN architecture $(\bm k, \bm s)$.
\rev{The main idea of the proof is:  
1) observe that the projective parametrization map $\nu^{\CC}$ is birational for reduced architectures; 
2) compute the singular points of its image (which is the projectivized complex function space) using the following method.}

\begin{fact}[{\cite[Lemma 3.2]{kohn2017secants}}]
\label{fact:birationalmap-sing}
Let $f:X \to Y$ be a birational finite surjective morphism between irreducible complex projective varieties and let $y \in Y$. The variety $Y$ is smooth at the point $y$ if and only if the fiber $f^{-1}(y)$ contains exactly one point $x \in X$, the variety $X$ is smooth at the point $x$ and the differential $d_xf : T_x X \to T_y Y$ is an injection.
\end{fact}

We begin by investigating the fibers of $\nu^\CC$.
For that, given the reduced sequence of strides $\bm s$, we define an \emph{$\bm s$-factorization} of a homogeneous polynomial $P \in \CC[x,y]$ to be a factorization of the form 
\begin{equation}\label{eq:s-factorization}
P(x,y) = Q_L(x^{S_L},y^{S_L}) \cdots  Q_2(x^{S_2},y^{S_2}) Q_1(x^{S_{1}},y^{S_{1}})
\end{equation}
such that, for all $l \in \lbrace L, \ldots, 2  \rbrace$, the factor $Q_{l-1}(x^{S_{l-1}},y^{S_{l-1}}) \cdots Q_{1}(x^{S_{1}}, y^{S_{1}}) $ is not divisible by any $S_l$-hyperroot.
Note that every homogeneous polynomial $P \in \CC[x,y]$ has an $\bm s$-factorization and that it is unique (up to scaling of the factors $Q_l$).
In particular, the degrees $d_l$ of the factors $Q_l$ in \eqref{eq:s-factorization} are uniquely determined by $P$ and $\bm s$.
We refer to the sequence $\rev{\bm d} = (d_L, \ldots, d_1)$ as the \emph{$s$-factor degrees of $P$}.
Those degrees give us a new perspective on LCN function spaces (see Lemma \ref{lem:s-factor-degrees-FS}) and the smaller function spaces they contain (see Corollary \ref{cor:strictContainment}),
which provides us with information on the fibers of $\nu^\CC$ (see Corollary \ref{cor:birational}).

\begin{lemma}
\label{lem:s-factor-degrees-FS}
    Let $(\bm k, \bm s)$ be a reduced LCN architecture. 
    Then 
    \begin{align*}
        \pi_1(\mathcal{M}_{\bm k, \bm s}^\CC) = \left\lbrace P \in \CC[x,y]_{k-1} : 
        \begin{array}{l}
        \text{the $\bm s$-factor degrees $\bm d$ of $P$ satisfy } \\ \sum_{i=l}^L d_i S_i \geq \sum_{i=l}^L (k_i-1)S_i    
\text{ for all } l = L,\ldots,1               
        \end{array}
\right\rbrace
    \end{align*}
\end{lemma}

\begin{proof}
% This is immediate from Proposition \ref{prop:functionspace-poly} and the definition of $\bm s$-factorization above.
\rev{By Proposition \ref{prop:functionspace-poly}, the set $\pi_1(\mathcal{M}_{\bm k, \bm s}^\CC)$ consists of polynomials $P \in \CC[x,y]_{k-1}$ of the form $P=P_L \cdots P_1$ with $P_i \in \CC[x^{S_i}, y^{S_i}]_{k_i-1}$. If $P = Q_L\cdots Q_1$ is the $\bm s$-factorization of any such polynomial, then for any $l=L,\ldots,1$ we have that $P_L\cdots P_{l}$ divides $Q_L\ldots Q_{l}$. This implies $\sum_{i=l}^L d_i S_i = \deg(Q_L\cdots Q_{l}) \ge \deg(P_L\cdots P_{l}) = \sum_{i=l}^L (k_i-1) S_i$. Conversely, if the $\bm s$-factor degrees of a polynomial $P$ satisfy these inequalities, then there exists at least one factorization of $P$ of the desired form.}
\end{proof}

\begin{corollary}
\label{cor:strictContainment}
    Let $\bm s $ be a reduced  sequence of strides, and let $\bm k, \bm k' \in \mathbb{Z}_{>0}^L$ be such that $\sum_{i=1}^L (k_i-1)S_i = \sum_{i=1}^L (k'_i-1)S_i$.
    Then $\mathcal{M}_{\bm k', \bm s}^\CC \subsetneq \mathcal{M}_{\bm k, \bm s}^\CC$ if and only if
    \begin{align}
        \label{eq:strictContainment} 
        \forall l = L,\ldots,2: \sum_{i=l}^L (k'_i-1) S_i \geq \sum_{i=l}^L (k_i-1)S_i,            
  \text{ and at least one inequality is strict.}         
    \end{align}
\end{corollary}

\begin{proof}
        We start by assuming \eqref{eq:strictContainment} and show the strict inclusion of function spaces.
        Let $P \in \pi_1(\mathcal{M}_{\bm k', \bm s}^\CC)$.
        By Lemma \ref{lem:s-factor-degrees-FS}, the $\bm s$-factor degrees $\bm d$ of $P$ satisfy 
        $\sum_{i=l}^L d_i S_i \geq \sum_{i=l}^L (k'_i-1)S_i $ for every layer $l$.
        Due to \eqref{eq:strictContainment}, this implies $\sum_{i=l}^L d_i S_i \geq \sum_{i=l}^L (k_i-1)S_i $ for every $l$, which means $P \in \pi_1(\mathcal{M}_{\bm k, \bm s}^\CC)$ (again by Lemma \ref{lem:s-factor-degrees-FS}).
        Hence, we have shown that $\mathcal{M}_{\bm k', \bm s}^\CC \subseteq \mathcal{M}_{\bm k, \bm s}^\CC$.
        We can see that the inclusion is strict, by considering a polynomial $P$ whose $\bm s$-factor degrees are exactly $(k_L-1, \ldots, k_1-1)$. 
        Then, $P \in \pi_1(\mathcal{M}_{\bm k, \bm s}^\CC)$.
        However, since one of the inequalities in \eqref{eq:strictContainment} is assumed to be strict, $P$ cannot be contained in $\pi_1(\mathcal{M}_{\bm k', \bm s}^\CC)$ due to Lemma \ref{lem:s-factor-degrees-FS}.

        Now we assume $\mathcal{M}_{\bm k', \bm s}^\CC \subsetneq \mathcal{M}_{\bm k, \bm s}^\CC$.
        This time, we consider a polynomial $P$ whose $\bm s$-factor degrees are exactly $(k'_L-1, \ldots, k'_1-1)$.
        Then $P \in \pi_1(\mathcal{M}_{\bm k', \bm s}^\CC) \subseteq \pi_1(\mathcal{M}_{\bm k, \bm s}^\CC)$, and so Lemma \ref{lem:s-factor-degrees-FS} implies that the inequalities in \eqref{eq:strictContainment} hold.
        Moreover, one of the inequalities has to be strict, because otherwise $\bm k'$ would be equal to $\bm k$, which would contradict that the inclusion $\mathcal{M}_{\bm k', \bm s}^\CC \subsetneq \mathcal{M}_{\bm k, \bm s}^\CC$ is strict.
\end{proof}

\begin{corollary} \label{cor:birational}
    Let $(\bm k, \bm s)$ be a reduced LCN architecture.
    If the fiber of a filter $w \in \PP(\mathcal{M}_{\bm k, \bm s}^\CC)$ under $\nu^\CC_{\bm k, \bm s}$ has cardinality larger one, 
    then $w$ is contained in a strictly smaller function space $\PP(\mathcal{M}_{\bm k', \bm s}^\CC) \subsetneq \PP(\mathcal{M}_{\bm k, \bm s}^\CC)$.
    In particular, the map $\nu^\CC_{\bm k, \bm s }$ is birational.
\end{corollary}

\begin{proof}
    Let $w \in \PP(\mathcal{M}_{\bm k, \bm s}^\CC)$ be a filter that is \emph{not} contained in any $\PP(\mathcal{M}_{\bm k', \bm s}^\CC)$ with $\PP(\mathcal{M}_{\bm k', \bm s}^\CC) \subsetneq \PP(\mathcal{M}_{\bm k, \bm s}^\CC)$.
    By Lemma \ref{lem:s-factor-degrees-FS} and Corollary \ref{cor:strictContainment}, the $\bm s$-factor degrees of the polynomial $P=\pi_1(w)$ are $(k_L-1, \ldots, k_1-1)$. 
    Hence, its $\bm s$-factorization in \eqref{eq:s-factorization} is the unique (up to scaling) factorization of $P$ according to the reduced architecture $(\bm k, \bm s)$.
    Therefore, the fiber of the corresponding filter $w$ under $\nu^\CC$ is a singleton.
    This proves the first assertion.

    Since complex LCN function spaces are irreducible varieties (as they are parametrized), the strict inclusion $\mathcal{M}^\CC_{\bm k', \bm s} \subsetneq \mathcal{M}^\CC_{\bm k, \bm s}$ is equivalent to that $\dim \mathcal{M}^\CC_{\bm k', \bm s} < \dim \mathcal{M}^\CC_{\bm k, \bm s}$.
    Thus, the first assertion of this corollary implies that the generic fiber of $\nu^\CC$ is a singleton.
    Therefore, $\nu^\CC$ is birational (see \cite[Exercise 7.8]{harris2013algebraic}).
\end{proof}

\begin{proposition}
    \label{prop:singComplex} 
    Let $(\bm k,\bm s)$ be a reduced LCN architecture with at least two layers.
    Then, we have that $\mathrm{Sing}(\mathcal{M}^\CC_{\bm k, \bm s}) = \lbrace 0 \rbrace \cup  \bigcup_{\bm k' \in K} \mathcal{M}^\CC_{\bm k', \bm s}$, where $K$ is the index set from Theorem \ref{thm:sing}.
\end{proposition}

\begin{proof}
    By Remark \ref{rmk:finite-morphism-nu},  every fiber of $\nu^{\mathbb{C}}$ has a finite number of elements. 
     Since $\nu^\CC$ is a morphism between projective varieties, this means that it is a finite morphism
 {\cite[III,§11, Exercise 11.2]{hartshorne}}.
 Moreover, $\nu^\CC$ is surjective (since $\overline{\mathcal{M}}^\CC = \mathcal{M}^\CC$) and birational (by Corollary \ref{cor:birational}).
 Hence, we can apply Fact \ref{fact:birationalmap-sing} to determine the singular points of $\PP(\mathcal{M}^\CC)$.
 Since the domain of $\nu^\CC$ is smooth, 
 we conclude that a filter $w$ is singular in $\PP(\mathcal{M}^\CC)$ if and only if $w$ is a critical value of $\nu^\CC$ or the fiber $(\nu^{\mathbb{C}})^{-1}(w)$ contains more than one element.
Each singular point of the latter type is contained in some $\PP(\mathcal{M}_{\bm k', \bm s}^\CC) \subsetneq \PP(\mathcal{M}_{\bm k, \bm s}^\CC)$ by Corollary \ref{cor:birational}.
For every critical value $w$ of $\nu^\CC$, 
 Theorem \ref{thm:crit-gene} (and Remark \ref{rmk:complex-proj-functionspace}) state that the $\bm s$-factor degrees $\bm d$ of the polynomial $\pi_1(w)$ satisfy  at least one of the inequalities $\sum_{i=l}^L d_i S_i \geq \sum_{i=l}^L (k_i-1)S_i$ (for $l=2, \ldots, L$) strictly,
 which shows that $w$ is also contained in some $\PP(\mathcal{M}_{\bm k', \bm s}^\CC) \subsetneq \PP(\mathcal{M}_{\bm k, \bm s}^\CC)$ by Corollary \ref{cor:strictContainment}.
 Hence, so far we have shown that $\mathrm{Sing}(\PP(\mathcal{M}^\CC_{\bm k, \bm s})) \subseteq \bigcup_{\bm k' \in K} \PP(\mathcal{M}^\CC_{\bm k', \bm s})$, where $K$ is the index set from Theorem \ref{thm:sing}.

 For the reverse inclusion, let us consider a filter $w \in \PP(\mathcal{M}^\CC_{\bm k', \bm s}) \subsetneq \PP(\mathcal{M}^\CC_{\bm k, \bm s})$.
Then,
 the $\bm s$-factor degrees $\bm d$ of the polynomial $P = \pi_1(w)$ satisfy  at least one of the inequalities $\sum_{i=l}^L d_i S_i \geq \sum_{i=l}^L (k_i-1)S_i$ (for $l=2, \ldots, L$) strictly. We fix the maximal such $l$, and consider the $\bm s$-factorization $P = Q_L \cdots Q_1$ in \eqref{eq:s-factorization}.
 Any factorization $P = P_L \cdots P_1$ according to the reduced architecture $(\bm k, \bm s)$ satisfies that $P_i=Q_i$ (up to scaling) for all $i > l$, that $P_l$ divides $Q_l$, and that $\frac{Q_l}{P_l}$ divides $P_{l-1}\cdots P_1$. 
 If the factor $Q_l$ is a power of an $S_l$-hyperroot, then $P_l$ and $P_{l-1}\cdots P_1$ have such a hyperroot in common, and so $w$ is a critical value of $\nu^\CC$ by Theorem \ref{thm:crit-gene} and thus a singular point of $\PP(\mathcal{M}^\CC_{\bm k, \bm s})$ by Fact \ref{fact:birationalmap-sing}.
 Otherwise, if $Q_l$ has at least two distinct $S_l$-hyperroots, then there are at least two distinct factorizations of $P$ according to the architecture $(\bm k, \bm s)$ (depending on which hyperroot is dividing $P_l$ and which $P_{l-1}\cdots P_1$).
 The latter means that the fiber of $w$ under $\nu^\CC_{\bm k, \bm s}$ has cardinality larger one, and so $w$ is a singular point of $\PP(\mathcal{M}^\CC_{\bm k, \bm s})$ by Fact \ref{fact:birationalmap-sing}.

Now we have shown that $\mathrm{Sing}(\PP(\mathcal{M}^\CC_{\bm k, \bm s})) = \bigcup_{\bm k' \in K} \PP(\mathcal{M}^\CC_{\bm k', \bm s})$. 
Since the architecture $(\bm k, \bm s)$ is reduced and has at least two layers, the zero filter is a singular point of the affine cone $\mathcal{M}^\CC_{\bm k, \bm s}$ by Lemma \ref{lem:zeroSing}. This proves the assertion.
\end{proof}

To prove Theorem \ref{thm:sing}, it remains to transfer Proposition \ref{prop:singComplex} to the real numbers and to understand when the index set $K$ is empty.
For the first, we investigate the $\bm s$-factorizations of real polynomials (see Lemma \ref{lem:sFactorsReal});
for the latter, we make use of a  technical statement on integers (see Lemma \ref{lem:integerRelation}).

\begin{lemma}
    \label{lem:sFactorsReal}
    All factors in the $\bm s$-factorization of a real homogeneous  polynomial $P \in \RR[x,y]$ are real.
\end{lemma}

\begin{proof}
    The factor $Q_L$ in \eqref{eq:s-factorization} is the product of all $S_L$-hyperroots of $P$. Those hyperroots are either real or appear in complex conjugated pairs by Lemma \ref{lem:hyperrootPairs}. 
    Hence, their product $Q_L$ is real.
    Since $Q_{L-1} \cdots Q_1$ is the $\bm s$-factorization of $\frac{P}{Q_L}$, we see inductively that all factors $Q_i$ are real. 
\end{proof}

\begin{corollary} \label{cor:zariskiClosureUnionReduced}
    Let $(\bm k, \bm s) $ be a reduced LCN architecture.
    Then $$\overline{\mathcal{M}}_{\bm k, \bm s} =   \bigcup_{\bm k' \in K \cup \{\bm k\}} \mathcal{M}_{\bm k', \bm s},$$ where $K$ is the index set from Theorem \ref{thm:sing}.
    In particular, $\overline{\mathcal{M}}_{\bm k, \bm s} \setminus {\mathcal{M}}_{\bm k, \bm s} \subseteq \bigcup_{\bm k' \in K } \mathcal{M}_{\bm k', \bm s}$.
\end{corollary}

\begin{proof}
    Since $\overline{\mathcal{M}}_{\bm k, \bm s}$ is the real part of $\mathcal{M}^\CC_{\bm k, \bm s}$, it corresponds to the set of real homogeneous polynomials whose $\bm s$-factor degrees $\bm d$ satisfy $\sum_{i=l}^L d_i S_i \geq \sum_{i=l}^L (k_i-1)S_i$ for all layers $l$.
    Thus, for all $\bm k' \in K$, we conclude that $\overline{\mathcal{M}}_{\bm k', \bm s} \subseteq \overline{\mathcal{M}}_{\bm k, \bm s}$.
    This shows the inclusion ``$\supseteq$'' in the assertion.
    For the other inclusion ``$\subseteq$'', let us consider a polynomial $P$ in the complement $\pi_1(\overline{\mathcal{M}}_{\bm k, \bm s} \setminus {\mathcal{M}}_{\bm k, \bm s})$.
    That means that $P$ does not have a real factorization according to the architecture $(\bm k, \bm s)$, but since its $\bm s$-factorization is real due to Lemma \ref{lem:sFactorsReal}, its $\bm s$-factor degrees $\bm d$ need to satisfy one of the inequalities $\sum_{i=l}^L d_i S_i \geq \sum_{i=l}^L (k_i-1)S_i$ strictly.
    Hence, $\bm k' := (d_1+1, \ldots, d_L+1) \in K$.
    Moreover, the $\bm s$-factorization of $P$ is a real factorization according to the architecture $(\bm k', \bm s)$, i.e., 
    and $P \in \pi_1(\mathcal{M}_{\bm k', \bm s})$.
\end{proof}

\begin{lemma} \label{lem:integerRelation}
    Let $k_1, \ldots, k_{l-1}$ and $S_1, \ldots, S_l$ be positive integers such that $S_i$ divides $S_{i+1}$ for all $i \in \lbrace 1, \ldots, l-1 \rbrace$.
    If $S_l \leq \sum_{i=1}^{l-1} (k_i-1)S_i$, then there are integers $0 \leq e_i < k_i$ for $i \in \lbrace 1, \ldots, l-1 \rbrace$ such that $S_l = \sum_{i=1}^{l-1}e_iS_i$.
\end{lemma}

\begin{proof}
    We prove the statement by induction on $l$.
    For $l=2$, we simply put $e_1 := \frac{S_2}{S_1}$.
    For $l>2$, we distinguish two cases.
    If $S_l \leq (k_{l-1}-1)S_{l-1}$, then we similarly set $e_{l-1}:= \frac{S_l}{S_{l-1}}$ and $e_i := 0$ for all $i < l-1$.
    Otherwise, we have 
    $0 < S_l - (k_{l-1}-1)S_{l-1} \leq \sum_{i=1}^{l-2} (k_i-1)S_i$ and we can apply the induction hypothesis to find integers $0 \leq e_i < k_i$ for $i < l-1$ such that $S_l - (k_{l-1}-1)S_{l-1} = \sum_{i=1}^{l-2} e_iS_i$.
    Setting $e_{l-1} := k_{l-1}-1$ concludes the proof.
\end{proof}

\begin{proof}[Proof of Theorem \ref{thm:sing}]
    Since $\overline{\mathcal{M}}$ is the real part of $\mathcal{M}^\CC$, 
    transferring Proposition \ref{prop:singComplex} to the real numbers yields  
    $\mathrm{Sing}(\overline{\mathcal{M}}_{\bm k, \bm s}) = \lbrace 0 \rbrace \cup  \bigcup_{\bm k' \in K} \overline{\mathcal{M}}_{\bm k', \bm s}$.
    Clearly, we have that $\bigcup_{\bm k' \in K} {\mathcal{M}}_{\bm k', \bm s} \subseteq \bigcup_{\bm k' \in K} \overline{\mathcal{M}}_{\bm k', \bm s}$.
    The reverse inclusion follows from Corollary \ref{cor:zariskiClosureUnionReduced}.
    It is left to show that the index set $K$ is empty if and only if every layer $l$ satisfies $S_l > \sum_{i=1}^{l-1} (k_i-1)S_i$.
    We start by assuming that $S_l > \sum_{i=1}^{l-1} (k_i-1)S_i$ holds for every $l$.
    We assume for contradiction that there is some $\bm k' \in K$.
    That means that one of the inequalities in \eqref{eq:strictContainment} is strict. 
    Let $l$ be the maximal layer  such that $\sum_{i=l}^L (k'_i-1) S_i > \sum_{i=l}^L (k_i-1)S_i$.
    Then we have that $k'_i = k_i$ for all $i > l$ and that $k'_l > k_l$.
    Since $\sum_{i=1}^L (k'_i-1) S_i = \sum_{i=1}^L (k_i-1) S_i$, we obtain that
    $\sum_{i=1}^{l-1} (k_i-1) S_i - \sum_{i=1}^{l-1} (k_i'-1) S_i
    = S_l (k'_l - k_l) > S_l > \sum_{i=1}^{l-1} (k_i-1) S_i$, which would imply that $0 > \sum_{i=1}^{l-1} (k'_i-1) S_i$; a contradiction. Hence, the set $K$ is empty. 
    Finally, we assume that some layer $l$ satisfies that $S_l \leq \sum_{i=1}^{l-1} (k_i-1) S_i$.
    We now set $k'_i := k_i$ for all $i > l$, $k'_l := k_l+1$, and $k'_j := k_j - e_j$ for all $j < l$, where the $e_j$ are the integers found in Lemma \ref{lem:integerRelation}.
    Then $\bm k' \in \mathbb{Z}_{>0}^L$ satisfies for every $j < l$ that
    $\sum_{i=j}^{L} (k'_i-1) S_i - \sum_{i=j}^{L} (k_i-1) S_i = S_l - \sum_{i=j}^{l-1} e_i S_i \geq$, with equality if $ j = 1$.
    This shows that $\bm k' \in K$, so the set $K$ is not empty.
\end{proof}

\section{The boundary of the function space}
\label{sec:boundary}

The boundary points of the stride-one LCN function spaces have been described in terms of the multiplicities of the real roots of polynomials in \cite{LCN}.
More specifically, if $P \in \RR[x,y]_{k-1}$ is a homogenous polynomial with $n$ distinct real roots with multiplicities $\alpha_1, \ldots, \alpha_n$, then Lemma 4.2 and Proposition 4.4 in \cite{LCN} state that
\begin{align*}
    P \in \pi_1(\mathcal{M}_{\bm k, \bm 1}) &\Leftrightarrow 
    \sum_{i=1}^n \alpha_i \geq e;  \\
    P \in \pi_1(\partial \mathcal{M}_{\bm k, \bm 1}) &\Leftrightarrow 
    \sum_{i=1}^n \alpha_i \geq e \text{ and }
    |\{ \alpha_i : \alpha_i \text{ is odd} \}| \leq e-2,
\end{align*}
where $e := |\{ k_i : k_i \text{ is even} \}|$.
Here, we provide a first extension of that result to LCN architectures with larger strides.
For that, we observe that, for every positive integer $s$,  every nonzero homogeneous polynomial $P \in \RR[x,y]$ can be uniquely factorized, up to scaling, into real homogeneous polynomials 
$P(x,y) = Q(x^s,y^s) \cdot R(x,y)$
such that $R \notin \RR[x^s,y^s]$.
Now the factorization of $Q$ into real irreducible factors yields 
\begin{equation*}
   P= \ell_1(x^s,y^s)^{\rho_1} \cdots \ell_r(x^s,y^s)^{\rho_r} \;\; \cdot \;\; q_1(x^s,y^s)^{\gamma_1}  \cdots q_c(x^s,y^s)^{\gamma_c}  \;\; \cdot \;\; R, 
\end{equation*}
where the linear $\ell_i$ and the quadratic $q_j$ correspond to the real and complex roots of $Q$.
We call the exponents
$(\rho_1, \ldots, \rho_r, \gamma_1, \ldots, \gamma_c )$ the \emph{real $s$-hyperroot multiplicities} of $P$.

\begin{theorem}
\label{thm:boundary2layer}
    Let $(\bm k, \bm s)$ be a two-layer reduced LCN architecture whose function space $\mathcal{M}_{{\bm k},{\bm s}}$ is not Zariski-closed.
    A nonzero filter $w \in \RR^k$ is in the relative boundary $\partial \FS$ if and only if the real $s_1$-hyperroot multiplicities of the polynomial $\pi_1({w})$ satisfy
    $\sum_{i=1}^{c} \gamma_{i} +  \sum_{j=1}^{r} \lfloor \frac{\rho_j}{2} \rfloor \ge \frac{k_2}{2}$ and $\sum_{j=1}^{r} \rho_j \ge 1$.
\end{theorem}

\begin{proof}
    By Theorem~\ref{thm:thick-closed}, the function space being non-Zariski closed is equivalent to the condition that $k_2$ is even and $s_1 \leq k_1-1$.
    Consider a sequence ${w^{(j)}}$ of filters from the set $\overline{\mathcal{M}}\setminus \mathcal{M}$. By Proposition \ref{prop:functionspace-poly}, each $w^{(j)}$ is associated with a polynomial $P^{(j)} \in \mathbb{R}[x,y]_{k-1}$ that can be factored as $P_2^{(j)}P_1^{(j)}$, where $P_1^{(j)} \in \mathbb{C}[x,y]_{k_1-1}$ and $P_2^{(j)} \in \mathbb{C}[x^{s_1},y^{s_1}]_{k_2-1}$. Note that this sequence is taken from outside of the function space, so no permutation of roots can make $P_1^{(j)}$ and $P_2^{(j)}$ real polynomials. This condition is equivalent to $P^{(j)}$ containing at least $k_2$ $s_1$-hyperroots that all are non-real
    (this is because $\deg(P_2^{(j)})=k_2-1$ is odd and the non-real hyperroots appear in complex conjugated pairs by Lemma \ref{lem:hyperrootPairs}). 

Now, suppose that the sequence ${P^{(j)}}$ converges to a polynomial $P$ that corresponds to a filter $w$ in the function space $\mathcal{M}$. Then $P$ should have a factorization $P_2P_1$ with $P_1 \in \mathbb{R}[x,y]_{k_1-1}$ and $P_2 \in \mathbb{R}[x^{s_1},y^{s_1}]_{k_2-1}$. Observe that for a polynomial in $\overline{\mathcal{M}}$, the existence of such a real factorization is equivalent to the existence of a real $s_1$-hyperroot. 
In the limit, the number of $s_1$-hyperroots (counted with multiplicity) cannot decrease, but real $s_1$-hyperroots can appear in two ways:
\begin{enumerate}[1)]
    \item 
A factor in $P_1^{(j)}$ that is not a $s_1$-hyperroot converges to a real $s_1$-hyperroot, which means $\rho_j\ge 1$ for at least one $j$. Then $P$ satisfies the relations $\sum_{j=1}^r \rho_j \ge 1$ and $\sum_{i=1}^c \gamma_i \ge \frac{k_2}{2}$.
\item 
A complex pair of $s_1$-hyperroots becomes real, i.e., $\rho_j\ge 2$ for at least one $j$. In this case, $\sum_{j=1}^r \rho_j \ge 2$ and $\sum_{i=1}^{c} \gamma_{i} + \sum_{j=1}^{r} \lfloor \frac{\rho_j}{2} \rfloor \ge \frac{k_2}{2}$.
\end{enumerate}
This shows that relative boundary points correspond to polynomials satisfying $\sum_{j=1}^r \rho_j \ge 1$ and $\sum_{i=1}^{c} \gamma_{i} + \sum_{j=1}^{r} \lfloor \frac{\rho_j}{2} \rfloor \ge \frac{k_2}{2}$.
Conversely, if a polynomial $P$ satisfies these inequalities,
we can construct a sequence in $\pi_1(\overline{\mathcal{M}} \setminus \mathcal{M})$ that converges to $P$  by replacing double real $s_1$-hyperroots with pairs of complex conjugated ones and the remaining real $s_1$-hyperroots  $R$ with $R +\epsilon^{(j)}xy$ such that $\epsilon^{(j)} \to 0$.  
\end{proof}

We can use the description of the boundary points for two layers in Theorem \ref{thm:boundary2layer}, 
to prove Theorem \ref{thm:boundaryProperties}c)
for arbitrarily many layers.
First, however, we show parts a) and b) of that theorem, and observe that the zero filter is on the relative boundary whenever the latter is non-empty.

\begin{lemma} \label{lem:zeroOnBoundary}
    Let $(\bm k, \bm s)$ be a  LCN architecture whose function space $\mathcal{M}_{{\bm k},{\bm s}}$ is not Zariski-closed.
    Then $0 \in \RR^k$ is contained in $\partial \FS$.
\end{lemma}
\begin{proof}
    If a polynomial $P$ belongs to the complement of the function space in its Zariski closure, we can generate a converging sequence $\frac{1}{n} P$, which also belongs to the complement and approaches zero in the limit.
\end{proof}

\begin{proof}[Proof of Theorem \ref{thm:boundaryProperties} a)-b)]
        We start with assertion (a) that compares the reduced boundary points $\partial \mathcal{M}_{\bm k, \bm s}^R \subseteq \partial \FS$ with the relative boundary $\partial \mathcal{M}_{\tilde{\bm k}, \tilde{\bm s}}$ of the reduced architecture $(\tilde{\bm k}, \tilde{\bm s})$.
    Both types of boundary points are limits of sequences in 
    $\overline{\mathcal{M}}_{{\bm k}, {\bm s}} \setminus \mathcal{M}_{\tilde{\bm k}, \tilde{\bm s}}$, with the difference that the limits land in $\FS$ or $\mathcal{M}_{\tilde{\bm k}, \tilde{\bm s}}$, respectively.
    This implies that 
    $\partial \mathcal{M}_{\bm k, \bm s}^R =  \partial \mathcal{M}_{\tilde{\bm k}, \tilde{\bm s}} \cap \FS.$

    For assertion (b), we recall from Corollary \ref{cor:zariskiClosureUnionReduced} that
    $\overline{\mathcal{M}}_{\bm k, \bm s} \setminus {\mathcal{M}}_{\tilde{\bm k}, \tilde{\bm s}} \subseteq \bigcup_{\tilde{\bm k}' \in \tilde{K} } \mathcal{M}_{\tilde{\bm k}', \tilde{\bm s}}$, where $\tilde{K} := K_{\tilde{\bm k}, \tilde{\bm s}}$ is the index set from Theorem \ref{thm:sing}. 
    Since LCN function spaces are Euclidean closed, limits of sequences in $\overline{\mathcal{M}}_{\bm k, \bm s} \setminus {\mathcal{M}}_{\tilde{\bm k}, \tilde{\bm s}}$ are also contained in the finite union of the $ \mathcal{M}_{\tilde{\bm k}', \tilde{\bm s}}$.
    In particular, this shows that $\partial \mathcal{M}_{\bm k, \bm s}^R \subseteq \bigcup_{\tilde{\bm k}' \in \tilde{K} } \mathcal{M}_{\tilde{\bm k}', \tilde{\bm s}}$.
    The latter union (which is empty in case the reduced architecture $(\tilde{\bm k}, \tilde{\bm s})$ has just one layer) is contained in $\mathrm{Sing}(\overline{\mathcal{M}}_{\bm k, \bm s})$ by Theorem \ref{thm:sing}.
\end{proof}

To estimate the dimension of the reduced boundary of an LCN function space, we use the two-layer description in Theorem \ref{thm:boundary2layer} to deduce necessary algebraic conditions the reduced boundary points have to satisfy.

\begin{proposition}
    \label{prop:boundaryReducedToTwo}
     Let $(\bm k, \bm s)$ be a reduced LCN architecture.
     For every nonzero relative boundary point $w \in \partial \FS$, 
     the associated polynomial $P=\pi_1(w)$
     satisfies one of the following:
     \begin{enumerate}
         \item $P$ has a real double $S_l$-hyperroot for some $l \in \{2, \ldots, L \}$, 
         \item the $\bm s$-factor degrees $\bm d$ of $P$ satisfy at least two of the inequalities $\sum_{i=l}^L d_i S_i \geq \sum_{i=l}^L (k_i-1)S_i$ for $l \in \{2, \ldots, L \}$ strictly, or
         \item  $\sum_{i=l}^L d_i S_i \geq \sum_{i=l}^L (k_i-1)S_i+2S_l$ for some $l \in \{2, \ldots, L \}$.
     \end{enumerate}
\end{proposition}

\begin{proof}
By Theorem \ref{thm:boundaryProperties} b), the $\bm s$-factor degrees $\bm d$ of $P$ satisfy at least one of the inequalities $\sum_{i=l}^L d_i S_i \geq \sum_{i=l}^L (k_i-1)S_i$ strictly.  
We now assume for contradiction the contrary of the assertion.
That means that exactly one of these inequalities is strict, say for the layer $l \in \{2, \ldots, L \}$, 
and that $\sum_{i=l}^L d_i S_i = \sum_{i=l}^L (k_i-1)S_i + S_l$, 
as well as that $P_l$ has no real double $S_i$-hyperroot for any $i \in \{2, \ldots, L \}$.

Since $P$ is taken from the relative boundary of the function space, 
there is a convergent sequence of polynomials $P^{(j)} \in \pi_1(\overline{\mathcal{M}}_{\bm k, \bm s}\setminus \mathcal{M}_{\bm k, \bm s}) $ with limit $P$.
Each real polynomial $P^{(j)}$ can be factored as $P_L^{(j)}\cdots P_1^{(j)}$, where $P_i^{(j)} \in \mathbb{C}[x^{S_i},y^{S_i}]_{k_i-1}$, and it does not admit an analogous real factorization due to Proposition \ref{prop:functionspace-poly}.
However, the limit polynomial $P$ has a real factorization $P_L \cdots P_1$ with $P_i \in \mathbb{R}[x^{S_i},y^{S_i}]_{k_i-1}$.

For any layer $i$, the number of $S_i$-hyperroots cannot decrease in the limit.
Hence, the $\bm s$-factor degrees $\bm d^{(j)}$ of $P^{(j)}$ satisfy
$\sum_{i=m}^L d_i S_i \geq \sum_{i=m}^L d^{(j)}_i S_i \geq \sum_{i=m}^L (k_i-1)S_i$ for every layer $m$.
Since we assumed that $\sum_{i=m}^L d_i S_i = \sum_{i=m}^L (k_i-1)S_i$ holds for all $m \neq l$,
we have 
\begin{align}
    \label{eq:degreesLarge}
d_m = d^{(j)}_m = k_m-1 \text{ for } m \in \{ L, \ldots, l+1\}, \\    
    \label{eq:degreesL}
d_l = k_l \geq d_l^{(j)} \geq k_l-1, \\
    \label{eq:degreesLminus1}
d_{l-1} = k_{l-1}-1-s_{l-1} \text{ and }
d_{l-1}^{(j)} +d_{l}^{(j)}s_{l-1} = d_{l-1} +d_{l}s_{l-1}, \\
    \label{eq:degreesSmall}
d_i = d^{(j)}_i = k_i-1 \text{ for } i \in  \{l-2, \ldots, 1\}.
\end{align}

In particular, if $L > l$, we see from $d_L = d_L^{(j)}$ that every $S_L$-hyperroot of $P$ is the limit of a sequence of $S_L$-hyperroots in $P^{(j)}$.
Moreover, $d_L = k_L-1$ implies that $P_L$ and $P_L^{(j)}$ are the products of the $S_L$-hyperroots in $P$ and $P^{(j)}$, respectively. 
Hence, $P_L^{(j)}$ converges to $P_L$ and $P_L^{(j)}$ is real by Lemma \ref{lem:sFactorsReal}.
Applying the same argument successively for $m = L, \ldots, l+1$ shows that $P_m^{(j)}$ is real and converges to $P_m$ for all $m > l$.
Analogously, we obtain from $d_{l-1}^{(j)} +d_{l}^{(j)}s_{l-1} = d_{l-1} +d_{l}s_{l-1} = (k_{l-1}-1) +(k_{l}-1)s_{l-1}$ that $P_l^{(j)}P_{l-1}^{(j)}$ is real and converges to $P_lP_{l-1}$.
Finally, we conclude from \eqref{eq:degreesSmall} that $P_i^{(j)}$ is real and converges to $P_i$ for all $i < l-1$.

Therefore, the factors $P_l^{(j)}$ and $P_{l-1}^{(j)}$ are non-real and 
their product $P_l^{(j)}P_{l-1}^{(j)}$ does not admit a real factorization of the same format.
This means that the sequence $P_l^{(j)}P_{l-1}^{(j)}$ comes from the complement $\overline{\mathcal{M}}_{\bm k',\bm s'} \setminus \mathcal{M}_{\bm k',\bm s'}$ of the two-layer architecture  $(\bm k', \bm s') = \left((k_{l-1}, k_l), (s_{l-1},1)\right)$ (after a change of variables $(x',y') := (x^{S_{l-1}},y^{S_{l-1}})$).
Moreover, it converges to $P_lP_{l-1}$ that corresponds to a filter inside the function space $\mathcal{M}_{\bm k',\bm s'}$.
By Theorem \ref{thm:boundary2layer}, the real $s_{l-1}$-hyperroot multiplicities of $P_lP_{l-1}$ satisfy $\sum_{i=1}^{c} \gamma_{i} +  \sum_{j=1}^{r} \lfloor \frac{\rho_j}{2} \rfloor \ge \frac{k_l}{2}$ and $\sum_{j=1}^{r} \rho_j \ge 1$.
Since we assumed that $P$ does not have any real double $S_l$-hyperroot, we obtain that $P_lP_{l-1}$ has at least $\frac{k_l}{2}$ pairs of complex conjugated $s_{l-1}$-hyperroots plus a  real one, meaning that it has at least $k_l+1$ many $s_{l-1}$-hyperroots in total.
In terms of the $\bm s$-factor degrees of $P$, this means that $d_l \geq k_l+1$.
This is a contradiction to \eqref{eq:degreesL}.
\end{proof}

The final ingredient for the proof of  Theorem \ref{thm:boundaryProperties} is the following technical statement on integers that we later use to estimate the dimension difference of LCN function spaces.

\begin{lemma}
    \label{lem:intergerDimension}
    Let $\bm s \in \mathbb{Z}_{>0}^{L-1}$, $S_l := \prod_{i=1}^{l-1} s_i$ for all $l=1, \ldots, L$, and let $\bm k, \bm k' \in \mathbb{Z}^L$ be distinct integer tuples such that
    $\sum_{i=1}^L (k'_i-1)S_i = \sum_{i=1}^L (k_i-1)S_i$ and
    $\sum_{i=l}^L (k'_i-1)S_i \geq \sum_{i=l}^L (k_i-1)S_i$ holds for all $l = 2, \ldots, L$.
    Then, 
    \begin{align}
        \label{eq:dimensionDiffLowerBound}
        \sum_{i=1}^L (k_i - k'_i) \geq \min \{ s_1, \ldots, s_{L-1} \} - 1.
    \end{align}
    Moreover, in the case that $s_i > 1$ for all $i$, the inequality \eqref{eq:dimensionDiffLowerBound} is strict if  one of the following conditions holds:
    \begin{enumerate}
        \item at least two of the inequalities $\sum_{i=l}^L (k'_i-1) S_i \geq \sum_{i=l}^L (k_i-1)S_i$ for $l \in \{2, \ldots, L \}$ are strict, or
        \item exactly one of the inequalities is strict and $k'_l > k_l+1$ for some $l \in \{2, \ldots, L \}$.
    \end{enumerate}
\end{lemma}

\begin{proof}
We prove the assertion by induction on $L$.
Since $\bm k$ and $\bm k'$ are assumed to be distinct, $L$ cannot be equal to one.
Thus, the base case of the induction is for $L=2$.
In that situation, we have $(k'_2-1)s_1+(k'_1-1) = (k_2-1)s_1 + (k_1-1)$ and $k'_2 > k_2$.
Hence, $k_2-k'_2+k_1-k'_1 = (k'_2-k_2)(s_1-1) \geq s_1-1$.
Moreover, in the case that $s_1>1$, the latter inequality is strict if and only if $k'_2 > k_2+1$.
Now we consider the case  $L>2$.
If the inequalities $\sum_{i=l}^L (k'_i-1)S_i \geq \sum_{i=l}^L (k_i-1)S_i$ are equalities for all $i \geq 3$, then $k'_i = k_i$ for $i \geq 3$ and $k'_2 > k_2$, which means that we can argue exactly as in the induction beginning.
If one of those inequalities for $i \geq 3$ is strict, we can apply the induction hypothesis to
$\overline{\bm s} := (s_2, \ldots, s_{L-1})$, 
$\overline{\bm k} := (k_2, \ldots, k_L)$, and
$\overline{\bm k}' := (k'_2 - \alpha, k'_3, \ldots, k'_L)$, where
we define $\alpha := \sum_{i=2}^L (k'_i-1)\frac{S_i}{s_1} - \sum_{i=2}^L (k_i-1)\frac{S_i}{s_1} \geq 0$.
This yields \begin{align}
    \label{eq:sumDifferenceAlpha}
    \sum_{i=2}^L (k_i - k'_i) + \alpha \geq \min \{ s_2, \ldots, s_{L-1} \} - 1
   \geq  \min \{ s_1, \ldots, s_{L-1} \} - 1.
\end{align}
Note that $\alpha s_1 = k_1 - k'_1$ due to 
$\sum_{i=1}^L (k'_i-1)S_i = \sum_{i=1}^L (k_i-1)S_i$.
Therefore, we have that 
\begin{align}\label{eq:lessLayersAlpha}
    \sum_{i=1}^L (k_i - k'_i) = \sum_{i=2}^L (k_i - k'_i) + \alpha + \alpha(s_1-1)
\geq \sum_{i=2}^L (k_i - k'_i) + \alpha,
\end{align}
and \eqref{eq:dimensionDiffLowerBound} follows from \eqref{eq:sumDifferenceAlpha}.
Finally, we assume that $s_i>1$ for all $i$.
If one of the two conditions in Lemma \ref{lem:intergerDimension} holds, then either $\alpha = 0$ and we see from applying the induction hypothesis that the first inequality in \eqref{eq:sumDifferenceAlpha} is strict, or $\alpha>0$ and the inequality in \eqref{eq:lessLayersAlpha} is strict.
In either case, \eqref{eq:dimensionDiffLowerBound} is strict.
\end{proof}

\begin{proof}[Proof of Theorem \ref{thm:boundaryProperties} c)]
    Since $\partial \mathcal{M}_{\bm k, \bm s}^R \subseteq \partial \mathcal{M}_{\tilde{\bm k}, \tilde{\bm s}}$ by Theorem \ref{thm:boundaryProperties} a), it is enough to show that
    \begin{align} 
    \label{eq:dimensionReducedBoundaryBounded}
        \dim \partial \mathcal{M}_{\tilde{\bm k}, \tilde{\bm s}} \leq \dim \FS - \min\{s_i: s_i > 1\}
        = \dim \overline{\mathcal{M}}_{\tilde{\bm k}, \tilde{\bm s}} - \min \{\tilde{s}_1, \ldots, \tilde{s}_{M-1} \}.
    \end{align}
    We first consider the case that $\min \{\tilde{s}_1, \ldots, \tilde{s}_{M-1} \} > \dim \overline{\mathcal{M}}_{\tilde{\bm k}, \tilde{\bm s}}$.
    This implies for every layer $m>1$ in the reduced architecture that $\tilde S_m = S_{m-1}s_{m-1} > S_{m-1}\sum_{i=1}^{m-1} (\tilde k_i-1) \geq \sum_{i=1}^{m-1} S_i (\tilde k_i-1)$.
    Thus, the relative boundary of $\mathcal{M}_{\tilde{\bm k}, \tilde{\bm s}}$ is empty by Theorem \ref{thm:thick-closed}c1).
    
    Hence, we assume in the following that $\min \{\tilde{s}_1, \ldots, \tilde{s}_{M-1} \} \leq \dim \overline{\mathcal{M}}_{\tilde{\bm k}, \tilde{\bm s}}$.
    We know that $\partial \mathcal{M}_{\tilde{\bm k}, \tilde{\bm s}} \subseteq \bigcup_{\tilde{\bm k}' \in \tilde{K} } \overline{\mathcal{M}}_{\tilde{\bm k}', \tilde{\bm s}} $ by Theorem \ref{thm:boundaryProperties} b).
   The dimension difference 
    \begin{align}
    \label{eq:dimDiffInProof}
         \dim \overline{\mathcal{M}}_{\tilde{\bm k}, \tilde{\bm s}} - \dim \overline{\mathcal{M}}_{\tilde{\bm k}', \tilde{\bm s}}
        = \sum_{i=1}^M (\tilde{k}_i - \tilde{k}'_i) 
        \geq  \min \{\tilde{s}_1, \ldots, \tilde{s}_{M-1} \}-1 
    \end{align}
    is estimated for all $\tilde{\bm k}' \in \tilde{K}$ in Lemma \ref{lem:intergerDimension}.
    
    We now distinguish between the three different types of boundary points described in Proposition \ref{prop:boundaryReducedToTwo}.
    The points in $\partial \mathcal{M}_{\tilde{\bm k}, \tilde{\bm s}}$ of the second type are contained in the union of $\overline{\mathcal{M}}_{\tilde{\bm k}', \tilde{\bm s}}$ whose dimension difference in \eqref{eq:dimDiffInProof} is strict (due to the first condition in the second part of Lemma \ref{lem:intergerDimension}).
    The points in $\partial \mathcal{M}_{\tilde{\bm k}, \tilde{\bm s}}$ that are not of the second type but of the third type in Proposition \ref{prop:boundaryReducedToTwo}
    are also contained in the union of $\overline{\mathcal{M}}_{\tilde{\bm k}', \tilde{\bm s}}$ whose dimension difference in \eqref{eq:dimDiffInProof} is strict (due to the second condition in Lemma \ref{lem:intergerDimension}).
    Thus, denoting the set of points in $\partial \mathcal{M}_{\tilde{\bm k}, \tilde{\bm s}}$ that are of the second or third type by $\partial \mathcal{M}^S_{\tilde{\bm k}, \tilde{\bm s}}$, we conclude that
    $\dim (\partial \mathcal{M}^S_{\tilde{\bm k}, \tilde{\bm s}}) \leq \dim \overline{\mathcal{M}}_{\tilde{\bm k}, \tilde{\bm s}} - \min \{\tilde{s}_1, \ldots, \tilde{s}_{M-1} \}$.

    Finally, we consider the set of points in $\partial \mathcal{M}_{\tilde{\bm k}, \tilde{\bm s}}$ whose associated polynomials have a real double $S_l$-hyperroot for some $S_l > 1$.
    We write $\Delta_{\tilde{\bm k}, \tilde{\bm s}}$ for the Zariski closure of that set in $\RR^k$.
    Since $\Delta_{\tilde{\bm k}, \tilde{\bm s}}  \subseteq \bigcup_{\tilde{\bm k}' \in \tilde{K} } \overline{\mathcal{M}}_{\tilde{\bm k}', \tilde{\bm s}} $, we have that 
    $\Delta_{\tilde{\bm k}, \tilde{\bm s}}  = \bigcup_{\tilde{\bm k}' \in \tilde{K} } \left(\Delta_{\tilde{\bm k}, \tilde{\bm s}} \cap
    \overline{\mathcal{M}}_{\tilde{\bm k}', \tilde{\bm s}} \right)$.
    By definition, $\Delta_{\tilde{\bm k}, \tilde{\bm s}}$ is contained in the discriminant hypersurface that describes polynomials with double roots.
    However, not every filter in $\overline{\mathcal{M}}_{\tilde{\bm k}', \tilde{\bm s}}$ corresponds to a polynomial with a double root, and so $\overline{\mathcal{M}}_{\tilde{\bm k}', \tilde{\bm s}}$ is not contained in the discriminant hypersurface.
    Hence, $\Delta_{\tilde{\bm k}, \tilde{\bm s}} \cap
    \overline{\mathcal{M}}_{\tilde{\bm k}', \tilde{\bm s}} \subsetneq \overline{\mathcal{M}}_{\tilde{\bm k}', \tilde{\bm s}}$.
    Since the latter is an irreducible variety, this shows that 
    $\dim (\Delta_{\tilde{\bm k}, \tilde{\bm s}} \cap
    \overline{\mathcal{M}}_{\tilde{\bm k}', \tilde{\bm s}})< \dim(\overline{\mathcal{M}}_{\tilde{\bm k}', \tilde{\bm s}})$.
    Therefore, we conclude  that
        $\dim \Delta_{\tilde{\bm k}, \tilde{\bm s}}  = \max_{\tilde{\bm k}' \in {\tilde{K}}} \dim (\Delta_{\tilde{\bm k}, \tilde{\bm s}} \cap \overline{\mathcal{M}}_{\tilde{\bm k}', \tilde{\bm s}})
        \leq \dim(\overline{\mathcal{M}}_{\tilde{\bm k}', \tilde{\bm s}}) - 1
        \leq \dim \overline{\mathcal{M}}_{\tilde{\bm k}, \tilde{\bm s}} - \min \{\tilde{s}_1, \ldots, \tilde{s}_{M-1} \}$, where the latter inequality comes from \eqref{eq:dimDiffInProof}.

        Since $\partial \mathcal{M}_{\tilde{\bm k}, \tilde{\bm s}} \subseteq \{ 0 \} \cup \partial \mathcal{M}^S_{\tilde{\bm k}, \tilde{\bm s}} \cup \Delta_{\tilde{\bm k}, \tilde{\bm s}}$, we have proven \eqref{eq:dimensionReducedBoundaryBounded}.
\end{proof}

\section{Optimization}
\label{sec:opt}

In this section, we prove Theorem \ref{thm:exposed}.
We start by expressing the squared error loss
$
\ell_{\mathcal D}(w)=\sum_{i=1}^N \|\texttt{y}^{(i)} - T_{w,s}\texttt{x}^{(i)}\|^2
$
directly in terms of the filters $w$ instead of first passing to Toeplitz matrices $T_{w,s}$.
For that, we collect the training data $\mathcal{D} = \{ (\texttt{x}^{(1)},\texttt{y}^{(1)}), \ldots, (\texttt{x}^{(N)},\texttt{y}^{(N)}) \} \subseteq \RR^{d_0} \times \RR^{d_L}$ into two matrices
$X \in \RR^{d_0\times N}$ and $Y \in \RR^{d_L\times N}$ whose columns are $\texttt{x}^{(1)}, \ldots, \texttt{x}^{(N)}$ and $\texttt{y}^{(1)}, \ldots, \texttt{y}^{(N)}$, respectively, 
and write $\ell_\mathcal{D}(w) = \| Y - T_{w,s}X \|^2_F$, where $\|\cdot\|_F$ is the Frobenius norm.
We next consider the linear map
\begin{align*}
    \chi_{X,s}: \RR^k \to \RR^{d_L \times N}, \quad w \mapsto T_{w,s} X.
\end{align*}
We let $\bar Y$ denote the orthogonal projection of $Y$ onto the image of $\chi_{X,s}$, and choose a filter $u_Y \in \RR^k$ such that $\chi_{X,s}(u_Y) = \bar Y$.
With this, we can write the squared error loss as
\begin{align}\label{eq:squaredLossOnFilters}
    \ell_\mathcal{D}(w)= \| Y - \bar Y \|^2_F + \| \bar Y - T_{w,s}X \|^2_F = \| Y - \bar Y \|^2_F + \| \chi_{X,s}(u_Y-w)  \|^2_F.
\end{align}
In this expression, $\bar Y$ and thus $\| Y - \bar Y \|^2_F$ only depend on the data $\mathcal{D}$ and $(k,s)$, but not on the filter $w$.
Hence, minimizing 
$\ell_\mathcal{D}(w)$ is equivalent to minimizing $\| \chi_{X,s}(u_Y-w)  \|_F^2$.
We observe that $\| \cdot \|_{X,s} := \| \chi_{X,s}(\cdot) \|_F$ is a seminorm on $\RR^k$ that is induced by the (possibly degenerate) inner product $\langle w_1,w_2 \rangle_{X,s} := \langle \chi_{X,s}(w_1), \chi_{X,s}(w_2) \rangle_F = \mathrm{tr}((T_{w_1,s}X)^\top T_{w_2,s}X) $.
\begin{lemma} \label{lem:seminorm}
    The inner product $\langle \cdot,\cdot \rangle_{X,s}$ is non-degenerate (i.e., $\| \cdot \|_{X,s}$ is a norm on $\RR^k$) if and only if the linear map $\chi_{X,s}$ is injective. 
    Moreover, if $N \geq k$, then $\chi_{X,s}$ is injective for almost all $X \in \RR^{d_0\times N}$.
\end{lemma}
\begin{proof}
    By definition of the inner product $\langle \cdot,\cdot \rangle_{X,s}$, we see that $\langle w,w\rangle_{X,s}=0$ if and only if $\chi_{X,s}(w)=0$, which shows the first part of the assertion.
    For the second part, we write $\bar{\texttt{x}}^{(i)} \in \RR^k$ for the first $k$ entries of the $i$-th column of $X$.
    The condition $\chi_{X,s}(w)=0$ implies in particular that $w \in \RR^k$ is orthogonal (with respect to the standard Euclidean inner product) to each of the vectors $\bar{\texttt{x}}^{(1)}, \ldots, \bar{\texttt{x}}^{(N)}$. 
    Hence, if $N\geq k$, almost all choices of $\bar{\texttt{x}}^{(1)}, \ldots, \bar{\texttt{x}}^{(N)}$ force $w$ to be zero. 
\end{proof}

\begin{corollary}\label{cor:squaredLossInnerProduct}
    Let $N \geq k$.
    For almost all $X \in \RR^{d_0\times N}$, minimizing the squared loss $\ell_\mathcal{D}(w)$ is equivalent to minimizing the squared inner product norm 
    $\| u_Y - w \|^2_{X,s} $, 
    where $u_Y$ is the unique filter such that $\chi_{X,s}(u_Y)$ is the orthogonal projection (with respect to the Frobenius norm) of $Y$ onto the image of $\chi_{X,s}$. 
\end{corollary}
\begin{proof}
    This follows immediately from \eqref{eq:squaredLossOnFilters} and Lemma \ref{lem:seminorm}.
\end{proof}

In the following, we fix an LCN architecture $(\bm k, \bm s)$ and write $\mu = \mu_{\bm k, \bm s}$
and $\langle \cdot,\cdot\rangle_X := \langle \cdot,\cdot \rangle_{X,s}$.
In light of Corollary \ref{cor:squaredLossInnerProduct}, we assume from now on that $N \geq k$ and that $X \in \RR^{d_0\times N}$ is such that minimizing  
$\ell_\mathcal{D}$ is equivalent to minimizing the squared norm $\| u_Y - \cdot \|^2_X$.
To prove Theorem \ref{thm:exposed}, it is now sufficient to show that, for a fixed $X$ and for almost every data filter $u \in \RR^k$, every  critical point of 
$\mathcal{L}_{u,X}(\theta) := \| u-\mu(\theta) \|^2_X$ satisfies one of the three conditions in Theorem \ref{thm:exposed}.

A filter tuple $\theta$ is a critical point of $\mathcal{L}_{u,X}$ if and only if the data filter $u$ is contained in the \emph{normal space} 
\begin{align}
\label{eq:normalSpace}
    N_X(\theta) := \mathrm{im}(d_\theta \mu)^{\perp_X} + \mu(\theta) \subseteq \RR^k,
\end{align}
where $\perp_X$ denotes the orthogonal complement with respect to the inner product $\langle \cdot,\cdot \rangle_X$.
Hence, to prove that a fixed set $\Theta$ of filter tuples does not contain any critical point of $\mathcal{L}_{u,X}$ for almost all data filters $u$, our proof strategy is to show that the union of the normal spaces $N_X(\theta)$ over all $\theta \in \Theta$ is contained in a proper algebraic subset of $\RR^k.$ 

\begin{definition}
  We say that a semialgebraic subset $\Theta \subseteq \mathbb{R}^{k_1}\times \ldots \times \mathbb{R}^{k_L}$ is \emph{exposed} with respect to $\langle\cdot,\cdot\rangle_X$ if 
  \begin{align*}
      \dim \left( \bigcup_{\theta \in \Theta} N_X(\theta) \right) = k.
  \end{align*}
\end{definition}

\begin{example}
    If $L \geq 2$, then $\mu^{-1}(0)$ is exposed.
Indeed, for any filter tuple $\theta $ with at least two zero filters, $d_\theta\mu=0$ and thus $\dim(N_X(\theta)) = k$. 
\end{example}

\begin{lemma}
    \label{lem:exposedAffineCone}
    Let $\Theta \subseteq (\mathbb{R}^{k_1} \setminus\{ 0\}) \times \ldots \times (\mathbb{R}^{k_L}\setminus\{ 0\})$ be a semialgebraic subset that is an affine cone, i.e., for every $\theta \in \Theta$ and $\lambda \in (\RR \setminus\{ 0\})^L$ we have that $(\lambda_1\theta_1, \ldots, \lambda_L\theta_L) \in \Theta$.
    Then $\dim \left( \bigcup_{\theta \in \Theta} N_X(\theta) \right)
    \leq \dim \Theta  - L + 1 + k - \min_{\theta \in \Theta} \rank(d_\theta\mu). 
    $
\end{lemma}
\begin{proof}
    Let us start by fixing a filter tuple $\theta\in\Theta$.
    For $\lambda \in (\RR \setminus\{ 0\})^L$, we write $\lambda \cdot \theta := (\lambda_1\theta_1, \ldots, \lambda_L\theta_L)$.
    Since  for every such $\lambda$ we have  
    $\mu(\lambda\cdot\theta) = \lambda_1\cdots\lambda_L \mu(\theta)$ and
    $\mathrm{im}(d_{\lambda \cdot \theta}\mu) = \mathrm{im}(d_\theta\mu)$,  we see that $\dim (\bigcup_{\lambda\in (\RR \setminus\{ 0\})^L} N_X(\lambda\cdot\theta)) \leq \dim N_X(\theta)+1$.
    Hence, 
    \begin{align*}
     &   \dim \left( \bigcup_{\theta \in \Theta} N_X(\theta) \right) = 
    \dim \left( \bigcup_{\bar\theta \in \PP(\Theta)} \bigcup_{\lambda\in (\RR \setminus\{ 0\})^L} N_X(\lambda\cdot\bar\theta) \right)
    \\&\leq \dim \PP(\Theta) + \max_{\bar\theta\in \PP(\Theta)} \dim \left(\bigcup_{\lambda\in (\RR \setminus\{ 0\})^L} N_X(\lambda\cdot\bar\theta) \right) 
    \leq \dim \Theta - L + \max_{\theta \in \Theta} (\dim N_X(\theta)+1).
    \end{align*}
    Since the dimension of the normal space $N_X(\theta)$ is $k - \rank(d_\theta\mu) $, the assertion follows.
\end{proof}

To estimate the rank of the differential of $\mu$, we investigate the $S_l$-hyperroots that the filters in a tuple $\theta$ have in common. 

\begin{definition} \label{def:chd}
    Let $\theta = (w_1, \ldots, w_L) \in \mathbb{R}^{k_1}\times \cdots \times \mathbb{R}^{k_L}$ with $w_l \neq 0$ and $P_l := \pi_{S_l}(w_l)$ for all $l$. 
    Set $G_l(\theta) := \gcd (P_l, P_{l-1} \cdots P_1) \in \RR[x^{S_l}, y^{S_l}]$ for $l \in \{2, \ldots, L \rbrace$, and 
    $G_1(\theta) := 1$.
    The \emph{common hyperroot degree} of $\theta$ is
    $\mathrm{chd}(\theta) := \sum_{l=2}^L \deg(G_l(\theta))$.
\end{definition}

Recall that Theorem \ref{thm:crit-gene}  states that $\theta$ with $\mu(\theta)\neq0$ is a critical point of $\mu$ if and only if $\mathrm{chd}(\theta) \geq 1$.

\begin{proposition} \label{prop:rankBound}
    For $\theta \in (\RR^{k_1} \setminus \{0\}) \times \cdots\times (\RR^{k_L} \setminus \{0\})$, we have that
    $\rank(d_\theta\mu) \geq \dim(\mathcal{M}) - \mathrm{chd}(\theta)$.
\end{proposition}
\begin{proof}
Using our identification $\pi_1$ of filters with polynomials, we consider $\theta = (P_1, \ldots, P_L)$ as a tuple of polynomials as in \eqref{eq:muPol}.  
For every layer $l$, we define $P'_l := \frac{P_l}{G_l(\theta)}$.
Then $P'_l$ and the product $P'_{l-1} \cdots P'_1$ are coprime.
Hence, $\theta' := (P'_1, \ldots, P'_L)$ is a regular point of
$\mu' := \mu_{\bm k', \bm s}$ where $k'_l := \deg P'_l+1$ 
(this follows from Theorem \ref{thm:crit-gene} after omitting all layers $l$ with $k'_l = 1$). 
We can see $d_{\theta'}\mu'$ as a restriction of $d_{\theta}\mu$ via the following commutative diagram:
\begin{center}
    \begin{tikzcd}
\prod_{i=1}^{L} \mathbb{R}[x^{S_i},y^{S_i}]_{k_i-1}  \arrow{r}{d_\theta\mu}  & \RR[x,y]_{k-1} \\   
\prod_{i=1}^{L} \mathbb{R}[x^{S_i},y^{S_i}]_{k'_i-1} \arrow{r}{d_{\theta'}\mu'}
\arrow[u,hook,"\varphi"] & \RR[x,y]_{k'-1}
\arrow[u,hook,"\psi"]
\end{tikzcd}
\end{center}
where $\varphi: (\dot P'_1, \ldots, \dot P'_L) \mapsto (\dot P'_1 G_1(\theta), \ldots, \dot P'_L G_L(\theta))$ and 
$\psi: \dot P' \mapsto \dot P' G_1(\theta)\cdots G_L(\theta)$.
Therefore, we conclude that
$\rank(d_\theta\mu) \geq \rank(d_{\theta'}\mu') = \dim(\mathcal{M}_{\bm k', \bm s}) = \sum_{i=1}^L (k'_i-1)+1
= \sum_{i=1}^L (k_i-1-\deg G_i(\theta))+1
= \dim(\FS) - \mathrm{chd}(\theta)$.
\end{proof}

We now aim to show that the critical points of $\mu$ (except those in $\mu^{-1}(0)$) are not exposed. 
For that, we stratify that set of critical points as the disjoint union (over  all $\delta \in \mathbb{Z}_{> 0}$) of 
$C_\delta := \lbrace \theta \in (\RR^{k_1} \setminus \{0\}) \times \ldots\times (\RR^{k_L} \setminus \{0\}) \mid \mathrm{chd}(\theta)=\delta \rbrace$.

\begin{proposition}\label{prop:dimCritBound}
    Let the architecture $(\bm k, \bm s)$ be reduced. If $\delta>0$ and  $C_\delta \neq \emptyset$, then we have that $\mathrm{codim}(C_\delta) > \delta$.
\end{proposition}
\begin{proof}
We prove the assertion by induction on the number $L$ of layers and -- as above -- identify filters with polynomials. For single-layer architectures, $\mathrm{chd}(\theta)=0$ for every non-zero $\theta$ and thus there is nothing to show.
Hence, the induction beginning is $L=2$. In that case, 
\begin{align*}
    C_\delta = \left\lbrace  (R Q_1, R Q_2) \;\middle\vert\; \begin{array}{l}
         R \in \RR[x^{s_1},y^{s_1}]_\delta, Q_2 \in \RR[x^{s_1},y^{s_1}]_{k_2-1-\delta}, Q_1 \in \RR[x,y]_{k_1-1-\delta s_1} \\
         R\neq0, Q_2 \neq 0, Q_1 \neq 0, \gcd(Q_2,Q_1) = 1
    \end{array} \right\rbrace
\end{align*}
has dimension $k_2+(k_1-\delta s_1)$. Therefore, the codimension of $C_\delta$ is $\delta s_1$, which is larger $\delta$ due to $s_1>1$ and $\delta>0$.

For the induction step, we assume $L>2$. 
For any partition $\bar \delta = (\delta_2,\ldots,\delta_L) \in \mathbb{Z}_{\geq 0}^{L-1}$ of $\delta$, we define 
$C_{\bar \delta} := \{ \theta \in C_\delta \mid \forall l = 2, \ldots, L: \deg G_l(\theta) = \delta_l \}$.
Then $C_\delta = \bigcup_{\bar \delta} C_{\bar\delta}$, where the union runs over all non-negative partitions of $\delta$ into $L-1$ parts.
Hence, it is enough to show for every such partition $\bar\delta$ that either $\mathrm{codim}(C_{\bar\delta}) > \delta$ or $C_{\bar\delta}=\emptyset$.
We fix a partition $\bar\delta$ with $C_{\bar\delta}\neq\emptyset$ and distinguish two cases.
First, if $\delta_L =0$, then for every $\theta' \in C_{(\delta_2, \ldots, \delta_{L-1})}$ we have that almost every $P_L \in \RR[x^{S_L},y^{S_L}]_{k_L-1}$ gives a point $(\theta', P_L) \in C_{\bar\delta}$.
Thus, the induction hypothesis yields 
$\dim C_{\bar\delta} = \dim C_{(\delta_2, \ldots, \delta_{L-1})} + k_L < (k_1+\cdots +k_{L-1}-\delta)+k_L$.

Second, if $\delta_L > 0$, we start by observing that 
\begin{align}\label{eq:CdeltaInequality}
    \dim C_{\bar \delta} < \dim C_{(\delta_2, \ldots, \delta_L-1)}.
\end{align}
Indeed, for every $(P_1, \ldots, P_{L-1},RQ_L) \in C_{(\delta_2, \ldots, \delta_L-1)}$ with $R \in \RR[x^{S_L},y^{S_L}]_{\delta_L-1}$ dividing the product $P_1\cdots P_{L-1}$,  almost every $Q'_L \in \RR[x^{S_L},y^{S_L}]_{k_L-\delta_L}$ yields a new point 
$(P_1,\ldots, P_{L-1},RQ'_L) \in C_{(\delta_2, \ldots, \delta_L-1)}$. 
Hence, every irreducible component of the Zariski closure $\overline{C}_{(\delta_2, \ldots, \delta_L-1)}$ is of the form $\Sigma \times \RR^{k_L-\delta_L}$.
However, no such component is contained in the Zariski closure $\overline{C}_{\bar\delta}$, because the latter imposes an algebraic condition on $Q'_L$. 
Since $\overline{C}_{\bar\delta}\subseteq \overline{C}_{(\delta_2, \ldots, \delta_L-1)}$ and no irreducible component of $\overline{C}_{(\delta_2, \ldots, \delta_L-1)}$ is equal to $\overline{C}_{\bar\delta}$, we have shown \eqref{eq:CdeltaInequality}.
Applying that inequality $\delta_L$ times, we obtain
$\dim C_{\bar\delta} \leq \dim C_{(\delta_2, \ldots, \delta_{L-1},0)}-\delta_L$.
Now, invoking the first case (where $\delta_L$ was assumed to be zero), we get
$\dim C_{(\delta_2, \ldots, \delta_{L-1},0)} < \sum_{i=1}^L k_i - \sum_{j=1}^{L-1} \delta_j$.
Putting the last two inequalities together, we  conclude
$\dim C_{\bar\delta}< \sum_{i=1}^L k_i - \sum_{j=1}^{L} \delta_j$.
\end{proof}

\begin{theorem}
\label{thm:CritNotExposed}
    If the architecture $(\bm k, \bm s)$ is reduced,
    $\mathrm{Crit}^\circ(\mu) := \lbrace \theta  \in \mathrm{Crit}(\mu) \mid \mu(\theta)\neq 0 \rbrace$ is not exposed.
\end{theorem}
\begin{proof}
    Since $\mathrm{Crit}^\circ(\mu)$ is the disjoint union over all $C_\delta$ with $\delta \geq 1$, we have that 
    \begin{align} \label{eq:exposedCritMax}
        \dim \left( \bigcup_{\theta \in \mathrm{Crit}^\circ(\mu)} N_X(\theta)\right)
        = \dim \left(\bigcup_{\delta \geq 1} \bigcup_{\theta \in C_\delta} N_X(\theta) \right)
        = \max_{\delta \geq 1}\; \dim \left( \bigcup_{\theta \in C_\delta} N_X(\theta)\right).
    \end{align}
    Applying Lemma \ref{lem:exposedAffineCone} to each non-empty $C_\delta$ in the union, we obtain
    $\dim \left( \bigcup_{\theta \in C_\delta} N_X(\theta) \right)
    \leq \dim C_\delta  - L + 1 + k - \min_{\theta \in C_\delta} \rank(d_\theta\mu).$
    Therefore, Propositions \ref{prop:rankBound} and \ref{prop:dimCritBound} yield that
    $\dim \left( \bigcup_{\theta \in C_\delta} N_X(\theta) \right)
    < (\sum_{i=1}^L k_i - \delta) - L + 1 + k - (\dim \mathcal{M} - \delta)=k.$
    Since $C_\delta$ is non-empty only for finitely many choices of $\delta$, the latter inequality shows that \eqref{eq:exposedCritMax} is less than $k$.
\end{proof}

\begin{corollary}
\label{cor:BoundarySingNotExposed}
    Let $\mathcal{Z} \subseteq \mathcal{M}$ be a semialgebraic subset that is an affine cone with $\dim(\mathcal{Z})< \dim(\mathcal{M})$.
    If the architecture $(\bm k, \bm s)$ is reduced, then $\mu^{-1}(\mathcal{Z} \setminus \{0\})$ is not exposed. 
    In particular, this statement holds for $\mathcal{Z} = \partial \mathcal{M}$ or $\mathcal{Z}= \mathcal{M}\cap\mathrm{Sing}(\overline{\mathcal{M}})$.
\end{corollary}
\begin{proof}
    $\Theta := \mu^{-1}(\mathcal{Z} \setminus \{0\})$ is the disjoint union of 
    $\Theta_C := \Theta \cap \mathrm{Crit}(\mu) = \Theta \cap \mathrm{Crit}^\circ(\mu)$
    and 
    $\Theta_R := \{ \theta \in \Theta \mid \theta \notin \mathrm{Crit}(\mu) \}$.
    By Theorem \ref{thm:CritNotExposed}, $\Theta_C$ is not exposed.
    Hence, it is left to show that $\Theta_R$ is not exposed either.
    Since every $\theta \in \Theta_R$ is a regular point of $\mu$, the rank of the differential $d_\theta\mu$ is equal to $ \dim \mathcal{M}$.
    Thus, applying Lemma \ref{lem:exposedAffineCone} to $\Theta_R$ yields
    $\dim \left( \bigcup_{\theta \in \Theta_R} N_X(\theta) \right)
    \leq \dim \Theta  - L + 1 + k - \dim \mathcal{M}
    = \dim \mathcal{Z} + k - \dim \mathcal{M} < k$.
\end{proof}

Theorem \ref{thm:CritNotExposed} and Corollary \ref{cor:BoundarySingNotExposed} imply Theorem \ref{thm:exposed} for reduced LCN architectures.
The general version follows from the following observation.

\begin{lemma}\label{lem:regularStride1NormalSpace}
    Let $\theta$ be a regular point of $(\mu_{\tilde {\bm k}^1, {\bm 1}} ,  \ldots , \mu_{\tilde {\bm k}^M, {\bm 1}})$. 
    Then $N_X(\theta) = N_X(\tilde \theta)$, where $\tilde \theta := (\mu_{\tilde {\bm k}^1, {\bm 1}} ,  \ldots , \mu_{\tilde {\bm k}^M, {\bm 1}})(\theta)$. 
\end{lemma}
\begin{proof}
    Since each $\mu_{\tilde {\bm k}^i, {\bm 1}}$ is the parametrization map of a stride-one LCN, its function space is thick. 
    Hence, at a regular point $\theta$ of $(\mu_{\tilde {\bm k}^1, {\bm 1}} ,  \ldots , \mu_{\tilde {\bm k}^M, {\bm 1}})$, the image of the differential of $(\mu_{\tilde {\bm k}^1, {\bm 1}} ,  \ldots , \mu_{\tilde {\bm k}^M, {\bm 1}})$
    is equal to the domain of the differential of $\mu_{\tilde{\bm k}, \tilde{\bm s}}$ at $\tilde \theta$.
    Therefore, \eqref{eq:stride-one-reduced-factorization} implies 
    $\mathrm{im}(d_\theta \mu_{\bm k, \bm s}) = \mathrm{im}(d_{\tilde \theta} \mu_{\tilde{\bm k}, \tilde{\bm s}})$, and the assertion follows.
\end{proof}

\begin{proof}[Proof of Theorem \ref{thm:exposed}]
    We consider the set of filter tuples $\theta$ that do not satisfy any of the three conditions in Theorem \ref{thm:exposed}.
    More concretely, writing 
    $$
    \Theta := \{ \theta \in \RR^{k_1}\times\cdots\times \RR^{k_L} \,\colon\, \mu_{\bm k,\bm s}(\theta) \neq 0, \theta \notin \mathrm{Crit}((\mu_{\tilde {\bm k}^1, {\bm 1}} ,  \ldots , \mu_{\tilde {\bm k}^M, {\bm 1}})) \},
    $$ 
    that set is the union of 
    $$\Theta_C := \Theta \cap \mathrm{Crit}(\mu_{\bm k, \bm s}), \;
    \Theta_S :=  \Theta \cap \mu^{-1}_{\bm k, \bm s}(\mathrm{Sing}(\overline{\mathcal{M}}_{\bm k, \bm s})), \; \text{and} \;
    \Theta_B :=  \Theta \cap \mu^{-1}_{\bm k, \bm s}(\partial \FS).$$
    It is sufficient to show that $\Theta_C \cup \Theta_S \cup \Theta_B$ is not exposed with respect to any $\langle \cdot,\cdot \rangle_X$ that is an inner product.
    For $\diamondsuit \in \{ C,S,B \}$,
    we set $\tilde\Theta_\diamondsuit := (\mu_{\tilde {\bm k}^1, {\bm 1}} ,  \ldots , \mu_{\tilde {\bm k}^M, {\bm 1}})(\Theta_\diamondsuit)$.
    Lemma \ref{lem:regularStride1NormalSpace} shows that
    $\bigcup_{\theta \in \Theta_\diamondsuit} N_X(\theta) = \bigcup_{\tilde\theta \in \tilde{\Theta}_\diamondsuit} N_X(\tilde\theta)$.
    Hence, it is enough to show that none of the $\tilde{\Theta}_\diamondsuit$ is exposed with respect to $\langle\cdot,\cdot\rangle_X$.
    Since $\tilde{\Theta}_C \subseteq \mathrm{Crit}^\circ(\mu_{\tilde{\bm k}, \tilde{\bm s}})$, we have that $\tilde{\Theta}_C$ is not exposed by Theorem \ref{thm:CritNotExposed}.
    Moreover,  $\tilde{\Theta}_S \cup \tilde{\Theta}_B \subseteq \mu^{-1}_{\tilde{\bm k}, \tilde{\bm s}}(\mathcal{Z} \setminus \{0\})$, where $\mathcal{Z} := (\FS \cap \mathrm{Sing}(\overline{\mathcal{M}}_{\bm k, \bm s})) \cup \partial\FS$.
    Thus, applying Corollary \ref{cor:BoundarySingNotExposed} to $\mathcal{Z} \subseteq \mathcal{M}_{\tilde{\bm k}, \tilde{\bm s}}$, we conclude that $\tilde{\Theta}_S \cup \tilde{\Theta}_B$ is not exposed.
\end{proof}

\subsection*{Acknowledgment}
KK was partially supported by the Wallenberg AI, Autonomous Systems and Software Program (WASP) funded by the Knut and Alice Wallenberg Foundation.
GM has been supported by NSF CAREER award 2145630, NSF award 2212520, DFG SPP~2298 grant 464109215, ERC Starting Grant 757983, and BMBF in DAAD project 57616814. 

\bibliographystyle{alpha}
\bibliography{literature}

\end{document}